\RequirePackage{etex}
\documentclass[12pt]{article}

\usepackage{natbib}
\PassOptionsToPackage{authoryear, round}{natbib}

\usepackage[margin=1in]{geometry}
\usepackage[utf8]{inputenc} 
\usepackage[T1]{fontenc}    
\usepackage[colorlinks, citecolor=blue]{hyperref}
\usepackage{url}            
\usepackage{booktabs}       
\usepackage{amsfonts}       
\usepackage{nicefrac}       
\usepackage{microtype}      
\usepackage{xcolor}         
\usepackage{comment}
\usepackage{amsmath,amsthm,amssymb,multirow,csquotes,cleveref}
\usepackage{bbm}
\usepackage{xhfill,tikz-network}
\usepackage{wrapfig}
\usepackage[export]{adjustbox}
\usepackage{graphicx,environ}
\usepackage{tikz}
\usepackage{tikzscale}
\usepackage{algorithm}
\usepackage{algorithmic}
\usepackage{pgfplots,enumitem,color,colortbl}
\usepackage{setspace}
\pgfplotsset{compat=newest}
\usetikzlibrary{automata,arrows,calc,positioning}

\DeclareMathOperator*{\argmin}{arg\,min}

\newcommand{\TMT}{\mathsf{T}}

\newcommand{\noise}{\varepsilon}
\newcommand{\cD}{{\mathcal{D}}}
\newcommand{\cB}{{\mathcal{B}}}
\newcommand{\cP}{{\mathcal{P}}}

\newcommand{\cY}{{\mathcal{Y}}}
\newcommand{\cT}{{\mathcal{T}}}
\newcommand{\cM}{{\mathcal{M}}}

\newcommand{\bhl}{\boldsymbol{\hat{\lambda}}}
\newcommand{\hl}{\hat{\lambda}}
\newcommand{\bl}{\boldsymbol{\lambda}}
\newcommand{\cX}{{\mathcal{X}}}
\newcommand{\normal}{{\mathcal{N}}}
\newcommand{\Sp}{{\mathbb{S}}}
\newcommand{\R}{{\mathbb{R}}}
\newcommand{\E}{{\mathbb{E}}}
\newcommand{\vx}{{\mathbf{x}}}
\newcommand{\va}{{\mathbf{a}}}

\newcommand{\vw}{{\mathbf{w}}}

\newcommand{\rf}{{\mathsf{RF}}}
\newcommand{\af}{\sigma}

\newtheorem{assumption}{Assumption}
\newtheorem{theorem}{Theorem}

\newtheorem{remark}{Remark}
\newtheorem{prop}{Proposition}

\newtheorem{definition}{Definition}

\usepackage{authblk}

\title{Aligning Model Properties via Conformal Risk Control}

\author[1]{William Overman\thanks{Email: woverman@stanford.edu}}
\author[2]{Jacqueline Jil Vallon\thanks{Email: jjvallon@alumni.stanford.edu}}
\author[1]{Mohsen Bayati\thanks{Email: bayati@stanford.edu}}

\affil[1]{Graduate School of Business, Stanford University}
\affil[2]{Management Science and Engineering, Stanford University}

\date{}

\begin{document}

\maketitle

\begin{abstract}
AI model alignment is crucial due to inadvertent biases in training data and the underspecified machine learning pipeline, where models with excellent test metrics may not meet end-user requirements. While post-training alignment via human feedback shows promise, these methods are often limited to generative AI settings where humans can interpret and provide feedback on model outputs. In traditional non-generative settings with numerical or categorical outputs, detecting misalignment through single-sample outputs remains challenging, and enforcing alignment during training requires repeating costly training processes.
In this paper we consider an alternative strategy. We propose interpreting model alignment through property testing, defining an aligned model $f$ as one belonging to a subset $\cP$ of functions that exhibit specific desired behaviors. We focus on post-processing a pre-trained model $f$ to better align with $\cP$ using conformal risk control. Specifically, we develop a general procedure for converting queries for testing a given property $\cP$ to a collection of loss functions suitable for use in a conformal risk control algorithm. We prove a probabilistic guarantee that the resulting conformal interval around $f$ contains a function approximately satisfying $\cP$. We exhibit applications of our methodology on a collection of supervised learning datasets for (shape-constrained) properties such as monotonicity and concavity. The general procedure is flexible and can be applied to a wide range of desired properties. Finally, we prove that pre-trained models will always require alignment techniques even as model sizes or training data increase, as long as the training data contains even small biases.
\end{abstract}

\section{Introduction}

The emergence of large foundation models has increased the attention to the problem of alignment. Aligned models are artificial intelligences designed to pursue goals that align with human values, principles, and intentions \citep{leike2018scalable, ouyang2022training, hendrycks2023aligning, ngo2024alignment}. Although the alignment problem is predominantly examined in the context of potential artificial general intelligence (AGI), large language models (LLMs), and reinforcement learning (RL) agents, it also has roots in the modern machine learning pipeline \citep{damour2022underspecification}. Motivated by this, we introduce a broader notion of alignment in this paper, extending beyond the aforementioned generative models to include even tabular regression models.

As an example, consider a regression task where property $\cP$ represents models that are monotonically decreasing in a given feature (covariate). For example, predicting cancer patient survival should be monotonically decreasing in cancer stage \citep{vallon2022clinically, vallon2024onaligning}. Constraining a prediction model during training to maintain monotonicity in this feature can be viewed as a form of alignment. For a pre-trained model $f$ that was trained without such constraints, ensuring monotonically decreasing predictions in this feature can be more complex. This complexity arises particularly in non-generative settings where the user cannot update $f$ or obtain any outputs other than point predictions $f(X)$ for a given input $X$.

In this work, we propose an approach to aligning a pre-trained model $f$ that is motivated by property testing \citep{Ron, Goldreich_2017} and conformal risk control \citep{angelopoulos2024conformal}. Property testing aims to design efficient algorithms for determining membership to the set $\cP$ of functions with a given property, that require fewer resources than learning algorithms for $\cP$ \citep{Ron}. This is particularly relevant for modern deep learning, where a user may need to determine if a pre-trained model $f$ belongs to $\cP$ without the resources to train a model of comparable size.

Property testing algorithms use local queries to determine, with high probability, whether a function has a given global property or is far from having it. We map such queries for a property $\cP$ to a set of loss functions, which we then use in a conformal risk control procedure \citep{angelopoulos2024conformal} to establish a notion of alignment for $\cP$. We prove that this procedure yields a conformal interval around $f$ containing a function close to $\cP$.

We demonstrate our methodology on real-world datasets for the properties of monotonicity and concavity. Motivated by the potential for systematic under- or over-estimation bias in $f$, we provide a straightforward extension of \citet{angelopoulos2024conformal} to obtain asymmetric conformal intervals with multi-dimensional parameters. While we examine both monotonicity and concavity constraints, the majority of our focus is on monotonicity, as these constraints have been shown to promote crucial aspects of alignment to human values, such as fairness and adherence to social norms \citep{wang2020deontological}.

While our methodology provides a way to align pre-trained models, one may question whether such techniques will remain necessary as AI capabilities advance. Given the outstanding capabilities of modern AI models with substantially large numbers of parameters and training data, one may argue that the alignment problem may naturally disappear as such advances continue \citep{kaplan2020scaling}. However, another contribution of this paper is to refute this argument in a stylized setting, building on recent advances in the theory of linearized neural networks \citep{mei2022generalization,misiakiewicz2023six}. Specifically, we show that increasing the size of the training data or the number of parameters in a random feature model (a theoretically tractable neural network proxy where hidden layer weights are randomly initialized and fixed \citep{rahimi2007random}) cannot help it satisfy a property $\cP$, if the pre-training data has biased labels. Our simulations show that the result holds even if only a small fraction of the training labels are impacted by the bias.

Summarizing our main contributions, we: (1) introduce an alignment perspective based on property testing, (2) use conformal risk control to post-process predictions of pre-trained models for better alignment, and (3) demonstrate that increasing training data and parameters in a random feature model does not eliminate the need for alignment. We discuss related work in Section \ref{sec:related}, particularly our connections to \citet{mitigating}, who use conformal risk control to address large language model hallucinations \citep{hallucinating}.

\section{Preliminaries}

In this section we provide key definitions drawn from property testing as well as a condensed overview to conformal prediction and conformal risk control. We provide a short introduction to propery testing in Appendix \ref{property-testing-appendix} and an extensive introduction to the field can be found in \cite{Goldreich_2017}. 

\subsection{Properties and Property Testing for Set-Valued Functions} \label{property-testing}

Our perspective on alignment in this work is motivated by the field of property testing \citep{Goldreich_2017,Ron}. Property testing studies algorithms that, by making a small number of local queries to a large object (such as a function or a graph), can determine whether the object has a certain property or is significantly far from having it.

Classic examples include linearity testing of Boolean functions \citep{blum}, testing whether a function is a low-degree polynomial \citep{lowdeg1,lowdeg2}, and testing $k$-juntas \citep{junta}. These algorithms generally operate by randomly sampling and querying the object, leveraging local information to infer global properties.

In this work, we focus on \emph{set-valued functions}, which are functions that map elements of a domain $\mathcal{X}$ to subsets of a codomain $\mathcal{Y}$, i.e., $F: \mathcal{X} \to 2^{\mathcal{Y}}$. While the standard definitions of property testing are technically sufficient for our purposes—since we can consider set-valued functions as functions with range $2^{\mathcal{Y}}$—we introduce specialized definitions to maintain clarity and to facilitate the transition between discussing $\mathcal{Y}$ and $2^{\mathcal{Y}}$.

\begin{definition}[Satisfying and Accommodating a Property]
Let \textbf{property} $\mathcal{P}$ denote a specific subset of all functions that map $\mathcal{X}$ to $\mathcal{Y}$. A function $f: \mathcal{X} \to \mathcal{Y}$ \textbf{satisfies} the property $\mathcal{P}$ if $f \in \mathcal{P}$.

A set-valued function $F: \mathcal{X} \to 2^{\mathcal{Y}}$ \textbf{accommodates} a property $\mathcal{P}$ if there exists a function $g \in \mathcal{P}$ such that $g(x) \in F(x)$ for all $x \in \mathcal{X}$.
\end{definition}

Intuitively, $F$ accommodates $\mathcal{P}$ if it contains at least one function $g$ satisfying $\mathcal{P}$ within its possible outputs.

We extend the notion of $\varepsilon$-farness from a property (as defined in Appendix~\ref{property-testing-appendix}) to set-valued functions. For set-valued functions, we measure the distance based on how often the outputs of any function $g \in \mathcal{P}$ fall within the sets provided by $F$.

\begin{definition}[$\varepsilon$-Faraway]
For a set-valued function \( F: \cX \to 2^\cY \), a distribution \( \cD \) over \( \cX \), \( \varepsilon > 0 \), and a property \( \cP \), we say \( F \) is \(\varepsilon\)-\textbf{Faraway} from \(\cP\) with respect to \(\cD\) if \( \delta_{\cP, \cD}(F) > \varepsilon \), where
\[
\delta_{\cP, \cD}(F) \overset{\text{def}}{=} \inf_{g \in \cP} \delta_{\cD}(F, g) \quad \text{and} \quad \delta_{\cD}(F, g) \overset{\text{def}}{=} \Pr_{X \sim \cD}[g(X) \not \in F(X)].
\]
\end{definition}

\textbf{Note.} Throughout this work, we assume that $\cD$ is the empirical distribution of a fixed and finite calibration dataset, and thus has finite support. While this assumption is not strictly necessary, most property testing results are over finite domains. Property testing over functions with Euclidean domains is a in general a difficult problem, though there have been notable recent successes \citep{fleming, arora}..

With these definitions in place, we can define \emph{testers} for set-valued functions. We focus on \emph{one-sided error testers}, which are algorithms that take in a set-valued function $F$, a distribution $\cD$, and a distance parameter $\varepsilon$ and output either Accept or Reject. These algorithms never reject a function that accommodates the property. The standard definition of one-sided error testers (provided in Appendix~\ref{property-testing-appendix}) extends naturally to set-valued functions by replacing the notion of satisfying a property with accommodating it.

\begin{definition}[One-Sided Error Tester for Set-Valued Functions]
A \emph{one-sided error tester} for a property $\mathcal{P}$ in the context of set-valued functions is a probabilistic oracle machine $\mathcal{M}$ that, given a distance parameter $\varepsilon > 0$, oracle access to a set-valued function $F : \mathcal{X} \to 2^\mathcal{Y}$, and oracle access to samples from a fixed but unknown distribution $\mathcal{D}$ over $\mathcal{X}$, satisfies:
\begin{enumerate}
    \item If $F$ accommodates $\mathcal{P}$, then $\Pr[\mathcal{M}^{F,\mathcal{D}}(\varepsilon) = \text{Accept}] = 1$.
    \item If $F$ is $\varepsilon$-Faraway from $\mathcal{P}$ with respect to $\mathcal{D}$, then $\Pr[\mathcal{M}^{F,\mathcal{D}}(\varepsilon) = \text{Accept}] \leq \frac{1}{3}$.
\end{enumerate}
Here, $\mathcal{M}^{F,\mathcal{D}}(\varepsilon)$ denotes the execution of the tester $\mathcal{M}$ when given oracle access to the function $F$, the distribution $\mathcal{D}$, and the parameter $\varepsilon$.
\end{definition}

Note that $\mathcal{M}$ itself is an abstract algorithm; $\mathcal{M}^{F,\mathcal{D}}$ is the instantiation of this algorithm with specific oracle access to $F$ and $\mathcal{D}$. 

In many property testing algorithms, the parameter $\varepsilon$ is used only to determine the number of iterations or samples required, not the core logic of the tester. This leads to the concept of \emph{proximity-oblivious testers} (POTs), where the basic testing procedure is independent of $\varepsilon$. The general definition of POTs (given in Appendix~\ref{property-testing-appendix}) also extends naturally to set-valued functions. 

\begin{definition}[Proximity-Oblivious Tester for Set-Valued Functions]\label{def:pot}
A \emph{proximity-oblivious tester} for a property $\mathcal{P}$ in the context of set-valued functions is a probabilistic oracle machine $\mathcal{T}$ that satisfies:
\begin{enumerate}
    \item If $F$ accommodates $\mathcal{P}$, then $\Pr[\mathcal{T}^{F,\mathcal{D}} = \text{Accept}] = 1$
    \item There exists a non-decreasing function $\rho : (0, 1] \to (0, 1]$ (called the \emph{detection probability}) such that if $F$ is $\varepsilon$-Faraway from $\mathcal{P}$,
    \[
    \Pr[\mathcal{T}^{F,\mathcal{D}} = \text{Reject}] \geq \rho( \varepsilon).
    \]
\end{enumerate}
Here, $\mathcal{T}^{F,\mathcal{D}}$ denotes the execution of the tester $\mathcal{T}$ when given oracle access to the function $F$ and the distribution $\mathcal{D}$.
\end{definition}

 To obtain a one-sided error tester with parameter $\varepsilon$, we can make $\Theta\left( \frac{1}{\rho(\varepsilon)} \right)$ independent calls to the POT $\mathcal{T}$ and accept if and only if all the calls accept \citep{pot}. We denote by $\mathcal{T}^{F,\mathcal{D}}(X)$ the output when applied to a specific sample $X \sim \cD$, and note that with abuse of notation we will late consider $\cD$ to be the empirical distribution of calibration dataset $\{(X_i,Y_i)\}_{i=1}^n$ in which case we write $\mathcal{T}^{F,\mathcal{D}}(X_i,Y_i)$ for the output on this specific sample from $\cD$.

\paragraph{Example.}
Consider functions $f: \mathbb{R}^d \to \mathbb{R}$, and let $\mathcal{P}$ denote the property that $f$ is constant in the $k$-th dimension. This property has has connections to fairness among other applications \cite{fairness}. Assume $\cD$ is the empirical distribution of the inputs $X \in \mathbb{R}$ for some fixed dataset.

Restrict to set-valued functions $F$ that output compact and connected intervals of the form $[a,b] \subseteq \mathbb{R}$ for $a,b \in \mathbb{R}$. The candidate POT $\mathcal{T}^{F, \mathcal{D}}$ for whether such a set-valued function $F$ accommodates $\cP$ is then as follows: sample $X, X' \sim \mathcal{D}$,  If $F(X) \cap F(X') \neq \varnothing$, then Accept; otherwise, Reject. We prove that this satisfies Definition ~\ref{def:pot} in Appendix~\ref{pot-proofs-const}.

\subsection{Conformal prediction and conformal risk control}

Our main tool for achieving alignment from this property perspective is built on conformal prediction and conformal risk control \citep{vovk, bates2021distributionfree, angelopoulos2024conformal}. Conformal prediction post-processes the outputs of any  model $f$ to create prediction intervals $C(\cdot)$ that ensure certain statistical coverage guarantees. Using a calibration dataset $\{(X_i, Y_i)\}_{i=1}^n$ consisting of ground truth input-output pairs, conformal prediction constructs intervals around the predictions of $f$ such that $\Pr[Y_{n+1} \notin C(X_{n+1})] \leq \alpha$ for a user-specified error rate $\alpha$ on a test point $(X_{n+1}, Y_{n+1})$.

This guarantee is notably distribution-free and holds for any function $f$. The probability is over the randomness in all $n+1$ points; both the calibration set and the test point. The construction of $C(\cdot)$ depends on both the model $f$ and the draw of the calibration data.

The conformal risk control framework extends conformal prediction to notions of error beyond miscoverage \citep{angelopoulos2024conformal}. Consider a paramater set $\Lambda \subset \mathbb{R}_{\geq 0}$ that is a bounded subset of the nonnegative reals. Given an exchangeable collection of non-increasing, random loss functions $L_i:\Lambda \to (-\infty, B]$, $i=1,\ldots,n+1$, conformal risk control uses the first $n$ loss functions and calibration data $\{(X_i, Y_i)\}_{i=1}^n$ to determine $\hat{\lambda}$ such that
\[
    \mathbb{E}[L_{n+1}(\hat{\lambda})] \leq \alpha.
\]

Consider loss functions of the form $L_i(\lambda) = \ell(C_\lambda(X_i), Y_i)$, where $C_\lambda(X_i)$ is a set of outputs constructed by $f$ and the calibration data. Larger values of $\lambda$ generate more conservative prediction sets $C_\lambda(\cdot)$. Let the risk on the calibration data for a given $\lambda$ be $\hat{R}_n(\lambda) = \frac{1}{n}\sum_{i=1}^n L_i(\lambda)$. For a user-specified risk rate $\alpha$, we let
\[
    \hat{\lambda} = \inf \left\{ \lambda : \frac{n}{n+1} \hat{R}_n(\lambda) + \frac{B}{n+1} \leq \alpha \right\}.
\]

This choice of $\hat{\lambda}$ guarantees the desired risk control $\mathbb{E}[L_{n+1}(\hat{\lambda})] \leq \alpha$ \citep{angelopoulos2024conformal}.

\section{Conformal property alignment}

Our main methodology is to use conformal risk control to create prediction intervals that align with specific properties $\mathcal{P}$. Our approach allows us to post-process the outputs of a pre-trained model $f$ to ensure that within the resulting conformal band, with a given probability, there exists predictions that adhere to desired properties such as monotonicity.

\subsection{Multi-lambda conformal risk control}
We make particular use of the conformal risk control algorithm to allow for a $k$-dimensional vector of tuning parameters $\boldsymbol{\lambda} = (\lambda_1, \lambda_2, \ldots, \lambda_k)$, where larger values of $\boldsymbol{\lambda} \in  \boldsymbol{\Lambda} \subset \mathbb{R}^k$ yield more conservative outputs, where $\boldsymbol{\Lambda} \subset \mathbb{R}_{\geq 0}^k$ is a bounded subset of $\mathbb{R}_{\geq 0}^k$. This works by  mapping $\bl$ to a scalar and then applying standard conformal risk control.  We emphasize that this result is not new and follows essentially directly from \cite{angelopoulos2024conformal}. The construction of the output set $F_{\boldsymbol{\lambda}}(X) \subseteq \cY$ depends on the specific application and provides flexibility in how the function $f(X)$ and the parameters $\boldsymbol{\lambda}$ are utilized.

\begin{definition}[Construction of $F_{\boldsymbol{\lambda}}(X)$]
Let $f: \mathcal{X} \to \mathcal{Y}$ be a given function. For each $\boldsymbol{\lambda} \in \mathbb{R}^k$, define the set-valued function $F_{\boldsymbol{\lambda}}: \mathcal{X} \to 2^{\mathcal{Y}}$ such that, for each $X \in \mathcal{X}$, $F_{\boldsymbol{\lambda}}(X)$ is a set of predictions for $X$ constructed from $f$ and $\boldsymbol{\lambda}$. The specific construction of $F_{\boldsymbol{\lambda}}(X)$ should satisfy the following properties:

\begin{enumerate}
    \item  When $\boldsymbol{\lambda} = \boldsymbol{0}$, we have $F_{\boldsymbol{0}}(X) = \{ f(X) \}$.
    \item For any $\boldsymbol{\lambda}, \boldsymbol{\lambda}' \in \mathbb{R}^k$, if $\boldsymbol{\lambda} \leq \boldsymbol{\lambda}'$ (i.e., $\lambda_i \leq \lambda_i'$ $\forall i = 1, 2, \ldots, k$), then $F_{\boldsymbol{\lambda}}(X) \subseteq F_{\boldsymbol{\lambda}'}(X)$.
\end{enumerate}
\end{definition}

This definition ensures that increasing the parameters $\boldsymbol{\lambda}$ leads to larger (more conservative) prediction sets, and that when all parameters are zero, the prediction set reduces to the point prediction given by $f(X)$.

Following the original scalar $\lambda$ setting, we assess $F_{\bl}$ using non-increasing random loss functions $L_i = \ell(F_{\bl}(X_i),Y_i) \in (-\infty, B]$ for $B < \infty$. In particular, we consider an exchangeable collection of non-increasing random functions $L_i: \boldsymbol{\Lambda} \to (-\infty, B], i=1,...,n+1$, where $\boldsymbol{\Lambda} \subset \mathbb{R}_{\geq 0}^k$ is a bounded subset of $\mathbb{R}_{\geq 0}^k$, with bound $\lambda^{\max}_j$ in each dimension $j \in [k]$.

As in \citet{angelopoulos2024conformal}, we use the first $n$ functions to determine $\bhl$ so that the risk on the $(n+1)$-th function is controlled, specifically so that $\E[L_{n+1}(\bhl)] \leq \alpha$.

We apply a similar algorithm. Given $\alpha \in (\infty, B)$ and letting $\hat{R}_n(\bl) = \frac{L_1(\bl) + \cdots + L_n(\bl)}{n}$, define

\[
\boldsymbol{\Lambda}_{\min} = \text{min} \left\{ \bl : \frac{n}{n+1} \hat{R}_n(\bl) + \frac{B}{n+1} \leq \alpha \right\}
\]

to be the set of minimal elements \citep{boyd2004convex} of $\boldsymbol{\Lambda}$ that satisfy the condition $\frac{n}{n+1} \hat{R}_n(\bl) + \frac{B}{n+1} \leq \alpha$. Let $g: \boldsymbol{\Lambda} \to \mathbb{R}$ be a strictly increasing function such that $L_i(\bl)$ is non-increasing with respect to the level sets defined by $g(\bl)$. Then select $\bhl \in \boldsymbol{\Lambda}_{\min}$ to be a minimizer of $g$ over $\boldsymbol{\Lambda}_{\min}$.

We then deploy the resulting set-valued function $F_{\bhl}$ on the test point $X_{n+1}$. For this choice of $\bhl$, we have a risk control guarantee that mimics the result of \citet{angelopoulos2024conformal}, specifically:

\begin{prop}\label{prop:1}
Assume that $L_i(\bl)$ is non-increasing with respect to the partial ordering of $\boldsymbol{\Lambda}$ inherited from $\mathbb{R}^k$. Additionally, assume that $L_i(\bl)$ is non-increasing with respect to $g(\bl)$ for some strictly increasing function $g: \boldsymbol{\Lambda} \to \mathbb{R}$. Also assume $L_i$ is right-continuous in each dimension, $L_i(\bl^{\max}) \leq \alpha$, and $\sup_{\bl} L_i(\bl) \leq B < \infty$ almost surely. Then 

\[
\E[L_{n+1}(\bhl)] \leq \alpha.
\]
\end{prop}

The proof is similar to the proof of the guarantee for the conformal risk control algorithm in \citet{angelopoulos2024conformal} and is deferred to Appendix \ref{multi-lambda}. 

To provide intuition on $g(\bl)$, we note that for our primary use case we will take $g(\bl)=\sum_{i=1}^k \lambda_i$. Clearly this function is strictly increasing in $\bl$ and intuitively it is reasonable to consider loss functions $L_i$ that are non-increasing as the sum of the components of $\bl$ increases.

\subsection{Conformal property alignment from proximity oblivious testers}\label{sec:mainthm}

We now demonstrate how to construct a conformal risk control problem using proximity-oblivious testers (POTs) for a given property $\mathcal{P}$. Suppose we are given a pre-trained model $f: \mathcal{X} \to \mathcal{Y}$. We aim to extend the point predictions of $f$ to prediction sets, where the size or conservativeness of the set is parameterized by a parameter $\bl$. Let $F_{\bl}: \mathcal{X} \to 2^{\mathcal{Y}}$ denote the set-valued function that outputs, for each $X \in \mathcal{X}$, the set $F_{\bl}(X) \subseteq \mathcal{Y}$ determined by $f$, $X$, and $\bl$.

Let $\mathcal{T}^{F,\mathcal{D}}$ be a proximity-oblivious tester for whether a set-valued function $F$ accommodates the property $\cP$ as given y Definition~\ref{def:pot}. We denote the random output of $\mathcal{T}^{F,\mathcal{D}}$ evaluated at $(X, Y) \sim \mathcal{D}$ by $\mathcal{T}^{F, \mathcal{D}}(X, Y)$.

We now define a loss function, generated from $\mathcal{T}^{F,\mathcal{D}}$, which will be crucial in formulating our conformal risk control problem.

\begin{definition}[Loss Function Generated from a POT]\label{def:loss_function}
Let $\mathcal{T}^{F, \mathcal{D}}$ be a proximity-oblivious tester for a property $\mathcal{P}$. We define the loss function $L_i$ as:
\[
L_i = 
\begin{cases}
0, & \text{if } \mathcal{T}^{F, \mathcal{D}}(X_i, Y_i) = \text{Accept}, \\
1, & \text{otherwise},
\end{cases}
\]
where $(X_i, Y_i)$ are samples from the distribution $\mathcal{D}$.
\end{definition}

\paragraph{Example.} Consider the POT for the property $\mathcal{P}$ of a function $f: \mathbb{R} \to \mathbb{R}$ being constant, as mentioned in Section~\ref{property-testing}. Assume we have access to a calibration set $\{(X_i, Y_i)\}_{i=1}^n$ of size $n$. We use a two-dimensional parameter $\bl = (\lambda^{-}, \lambda^{+})$, and define the set-valued function:
\[
F_{\bl}(X) = [f(X) - \lambda^{-},\ f(X) + \lambda^{+}],
\]
for each $X \in \mathbb{R}$. This creates prediction intervals around the point prediction $f(X)$, with widths controlled by $\lambda^{-}$ and $\lambda^{+}$.

We then apply the loss function generated by $\mathcal{T}^{F_{\lambda}, \mathcal{D}}$ as given in Definition~\ref{def:loss_function}, and use conformal risk control to tune $\bl$ such that the expected loss on the $(n+1)$th point falls below a given target level $\alpha$.

Note that in this case the tester and loss function does not depend on the $Y_i$. This is because the property of $f$ being constant does not depend on the $Y_i$ from the calibration set and here $\cD$ is only used to obtain samples of the $X_i$. This is not the case in general, however, and properties can be defined with respect to the whole sample $(X_i,Y_i) \sim \cD$. For example, we could consider the property $\cP$ that $f$ does not over-predict, that is, for $(X,Y) \sim \cD$ we have $f(X) \leq Y$. Now we state our main theorem.

\begin{theorem} \label{thm:main}
Let \( \mathcal{T} \) be a proximity-oblivious tester for a property \( \mathcal{P} \) with detection probability function \( \rho(\cdot) \). Assume access to a calibration dataset \(\{(X_i, Y_i)\}_{i=1}^n\) sampled independently from a distribution \( \mathcal{D} \). Suppose we run conformal risk control on this calibration dataset using risk parameter \( \alpha \) and loss functions \( L_i \) for property $\cP$ generated from $\cT$ (as in Definition~\ref{def:loss_function}). Then, for any \( \varepsilon \) such that \( \rho(\varepsilon) > \alpha \), the probability that \( F_{\hat{\boldsymbol{\lambda}}} \) is \( \varepsilon \)-Faraway from \( \mathcal{P} \) satisfies:
\[
\Pr_{(X_1,Y_1),...,(X_n,Y_n)}\left( F_{\hat{\boldsymbol{\lambda}}} \text{ is } \varepsilon\text{-Faraway from } \mathcal{P} \right) \leq \frac{\alpha}{\rho(\varepsilon)}.
\]
\end{theorem}

\begin{proof}
Let \( \mathcal{E} \) denote the event that \( F_{\hat{\boldsymbol{\lambda}}} \) is \( \varepsilon \)-Faraway from the property \( \mathcal{P} \). Our goal is to bound the probability \( \Pr_{(X_1,Y_1),...,(X_n,Y_n)}[\mathcal{E}] \).

The conformal risk control procedure ensures that the expected loss on a new sample \( (X_{n+1}, Y_{n+1}) \) satisfies:
\[
\mathbb{E}_{(X_1,Y_1),...,(X_n,Y_n), (X_{n+1}, Y_{n+1})}[ L_{n+1} ] \leq \alpha.
\]
Now, we can write
\begin{align*}
\mathbb{E}[ L_{n+1} ] &= \Pr( \mathcal{E} ) \cdot \mathbb{E}[ L_{n+1} \mid \mathcal{E} ] + \Pr( \mathcal{E}^c ) \cdot \mathbb{E}[ L_{n+1} \mid \mathcal{E}^c ].
\end{align*}

When \( \mathcal{E} \) occurs, \( F_{\hat{\boldsymbol{\lambda}}} \) is \( \varepsilon \)-Faraway from \( \mathcal{P} \). By the properties of the proximity-oblivious tester \( \mathcal{T} \), we have:
\[
\Pr_{ (X, Y) \sim \mathcal{D} }\left[ \mathcal{T}^{ F_{\hat{\boldsymbol{\lambda}}}, \mathcal{D} }( X, Y ) = \text{Reject} \mid \mathcal{E} \right] \geq \rho( \varepsilon ).
\]
Thus, the conditional expected loss satisfies:
\[
\mathbb{E}[ L_{n+1} \mid \mathcal{E} ] = \Pr\left[ \mathcal{T}^{ F_{\hat{\boldsymbol{\lambda}}}, \mathcal{D} }( X_{n+1}, Y_{n+1} ) = \text{Reject} \mid \mathcal{E} \right] \geq \rho( \varepsilon ).
\]

And when since $\Pr( \mathcal{E}^c ) \cdot \mathbb{E}[ L_{n+1} \mid \mathcal{E}^c ]$ is non-negative because $L_{n+1}$ is non-negative, we obtain
\begin{align*}
\mathbb{E}[ L_{n+1} ] &\geq \Pr( \mathcal{E} ) \cdot \rho( \varepsilon ).
\end{align*}

Combining this result with our guarantee from the conformal risk control procedure, 
\[
\alpha \geq \mathbb{E}[ L_{n+1} ] \geq \Pr( \mathcal{E} ) \cdot \rho( \varepsilon ).
\]
This implies:
\[
\Pr( \mathcal{E} ) \leq \frac{ \alpha }{ \rho( \varepsilon ) }.
\]

Therefore, the probability that \( F_{\hat{\boldsymbol{\lambda}}} \) is \( \varepsilon \)-Faraway from \( \mathcal{P} \) satisfies:
\[
\Pr_{(X_1,Y_1),...,(X_n,Y_n)}\left( F_{\hat{\boldsymbol{\lambda}}} \text{ is } \varepsilon\text{-Faraway from } \mathcal{P} \right) \leq \frac{ \alpha }{ \rho( \varepsilon ) }.
\]
\end{proof}

\paragraph{Amplifying Detection Probability via Independent Calls.}
\label{sec:amplify_detection_probability}
When the detection probability $\rho(\varepsilon)$ of the proximity-oblivious tester $\mathcal{T}$ is less than or close to the risk parameter $\alpha$, the bound provided by Theorem~\ref{thm:main} may not be tight or meaningful (since $\alpha / \rho(\varepsilon)$ could be greater than or equal to 1). To address this issue, we can amplify the detection probability by performing multiple independent executions of $\mathcal{T}$ and combining their results appropriately.

To increase the detection probability beyond $\alpha$, we execute the proximity-oblivious tester $\mathcal{T}$ independently $k$ times on independent samples and define a new tester $\mathcal{T}'$ that rejects if \emph{any} of the $k$ executions reject (i.e., by applying a logical OR to the outcomes). This amplification technique yields an adjusted detection probability 
\[
\rho'(\varepsilon) = 1 - (1 - \rho(\varepsilon))^k,
\]
representing the probability that at least one of the $k$ independent executions rejects when the function is $\varepsilon$-Faraway from $\mathcal{P}$.

In this approach, the calibration dataset needs to be partitioned into $n' = \left\lfloor \frac{n}{k} \right\rfloor$ disjoint batches, each containing $k$ samples. Each batch provides the independent samples required for the $k$ executions of $\mathcal{T}$ per calibration point. As a result, the effective sample size available for calibrartion becomes $n'$ due to this batching of samples.

\section{Examples}

\subsection{Monotonicity}\label{sec:monotonicity}

Monotonic behavior is important in various applications. We focus on monotonicity in a single feature, where we expect that $f(X)$ should have monotonically increasing or decreasing behavior with respect to a certain feature $x^k$ when other features $x^{-k}$ are held fixed.  While there is a long-standing literature on using monotonic constraints for regularization \citep{Brunk1973StatisticalIU, sill, you2017deep, bonakdarpour2018prediction} and on integrating such monotonic shape constraints into prediction models \citep{groeneboom_nonparametric_2014, cano2018monotonic, runje2023constrained}, our aim is not to view monotonicity as a possible means to improve test accuracy, but rather as a user-desired property for safe or fair deployment of a given model. For example, \citet{wang2020deontological} highlight the importance of monotonicity in models for criminal sentencing, wages, and medical triage.

Consider a user given a pre-trained model $f$ that was not trained with monotonic constraints. The user, however, wishes for the sake of safe or fair deployment to make predictions in a way that is as monotonic as possible. 
In particular, let $\cP$ be the property that $f$ is monotonically decreasing in dimension $k$. To apply our methodology we consider the proximity oblivious tester $\cT$ for $\cP$ as given in Algorithm \ref{alg:POT_monotonicity}.

\begin{algorithm}
\caption{POT $\cT$ for property $\cP$ of monotonically decreasing in dimension $k$}
\label{alg:POT_monotonicity}
\begin{algorithmic}[1]
\STATE Sample $X_1 \sim \mathcal{D}$. Let $X_1=(x_1, x^{-k})$
\STATE Sample $x_2$ from the marginal distribution of $\cD$ in dimension $k$. Set $X_2=(x_2, x^{-k})$
\IF{$x_1 < x_2$ and $\max{F(X_1)} < \min F(X_2)$}
    \RETURN Reject
\ELSIF{$x_2 < x_1$ and $\max F(X_2) < \min{F(X_1)}$}
    \RETURN Reject
\ENDIF
\RETURN Accept
\end{algorithmic}
\end{algorithm}

 We prove in Appendix~\ref{pot-proofs-mon} that Algorithm \ref{alg:POT_monotonicity} is indeed a POT for the property $\cP$ of being monotonically decreasing in a given dimension. Then let $\cM$ be the one-sided error tester for $\cP$ resulting from $\Theta(1/\rho(\varepsilon))$ calls to $\cT$. Now assume we have access to a calibration dataset $\{(X_i, Y_i)\}_{i=1}^n$ sampled from $\cD$ of size $n \in \Omega(1/\rho(\varepsilon))$. We will use this calibration dataset to determine the setting of $\bl = (\lambda^+, \lambda^-)$ via conformal risk control where the loss function is generated as in Definition~\ref{def:loss_function}. Here the set-valued function will be constructed as $F_{\bl}(X) = [f(X) - \lambda^-, f(X) + \lambda^+]$. Then by Theorem~\ref{thm:main} if the tester has sufficient detection probability $\rho(\varepsilon) > \alpha$ we expect to obtain a set-valued function $F_{\bhl}$ at most $\varepsilon$ from $\cP$. We now investigate this empirically.

\paragraph{Setup.} We align for monotonicity on various UCI ML repository datasets \citep{ucimlrepo} with a 70-15-15 train-calibrate-test split, averaged over 30 random splits. We use XGBoost regression models \citep{Chen_2016}. For each dataset, we select a feature for which we desire the model to be monotonic, not with the intention of improving test-set accuracy, but from the perspective of a user who desires this property.

We train two models per dataset: one unconstrained, trained on the training set, and another constrained to be monotonic, trained on both the training and calibration sets. The conformal risk control procedure is applied to the unconstrained model using the calibration data. The constrained model can be considered best possible from the user's perspective, using all available pre-test data and satisfying the monotonicity property $\cP$ during training.

To compare performance with respect to the training metric of accuracy, we convert conformal intervals into point predictions by taking $k$-quantiles of the constrained feature, linearly interpolating between adding $\lambda^+$ at the lowest quantile to subtracting $\lambda^-$ at the highest quantile for monotonically decreasing, or vice versa for monotonically increasing.

\paragraph{Results.} Table \ref{table:power} presents results on the test set for the Combined Cycle Power Plant dataset \citep{power_plant}. In practice, Exhaust-vacuum is known to negatively influence turbine efficiency \citep{power_plant}. The conformal procedure outperforms the constrained model in terms of MSE for all $\alpha$, which is a fortuitous but unexpected outcome. The constrained model should be seen as an oracle benchmark in the sense that the model was given to the user already trained to satisfy the desired property. The results on this dataset and  The risk metric closely matches the theoretical guarantee from conformal risk control and achieves optimal performance of 0 for the constrained model. Additional datasets and results are detailed in the appendix.

\begin{table}[h!]\label{table:power}
  \centering
  \caption{Power Plant, $n=9568$. Monotonically decreasing on Exhaust Vacuum. $\bl^{\max}=(10,10)$.}
  \begin{tabular}{c@{\hspace{12pt}}l@{\hspace{12pt}}c@{\hspace{12pt}}c@{\hspace{12pt}}c@{\hspace{12pt}}c}
    \toprule
    $\alpha$ & $\bl$ & Metric & Unconstrained & Adjusted & Constrained \\
    \midrule
    & $\lambda^+ = 0.51_{\scriptsize{(\pm 0.24)}}$ & MSE & $10.19_{\scriptsize{(\pm 0.46)}}$ & $10.47_{\scriptsize{(\pm 0.46)}}$ & $16.21_{\scriptsize{(\pm 0.45)}}$ \\
    0.1 & $\lambda^- = 0.76_{\scriptsize{(\pm 0.24)}}$ & Risk & $0.75_{\scriptsize{(\pm 0.09)}}$ & $0.10_{\scriptsize{(\pm 0.001)}}$ & $0.00_{\scriptsize{(\pm 0.00)}}$ \\
    \midrule
    & $\lambda^+ = 1.09_{\scriptsize{(\pm 0.51)}}$ & MSE & $10.19_{\scriptsize{(\pm 0.46)}}$ & $11.42_{\scriptsize{(\pm 0.44)}}$ & $16.21_{\scriptsize{(\pm 0.45)}}$ \\
    0.05 & $\lambda^- = 1.61_{\scriptsize{(\pm 0.50)}}$ & Risk & $0.75_{\scriptsize{(\pm 0.09)}}$ & $0.05_{\scriptsize{(\pm 0.001)}}$ & $0.00_{\scriptsize{(\pm 0.00)}}$ \\
    \midrule
    & $\lambda^+ = 2.39_{\scriptsize{(\pm 0.82)}}$ & MSE & $10.19_{\scriptsize{(\pm 0.46)}}$ & $14.46_{\scriptsize{(\pm 0.48)}}$ & $16.21_{\scriptsize{(\pm 0.45)}}$ \\
    0.01 & $\lambda^- = 3.33_{\scriptsize{(\pm 0.79)}}$ & Risk & $0.75_{\scriptsize{(\pm 0.09)}}$ & $0.01_{\scriptsize{(\pm 0.001)}}$ & $0.00_{\scriptsize{(\pm 0.00)}}$ \\
    \bottomrule
  \end{tabular}
\end{table}

\begin{figure}[ht]
    \centering
    \includegraphics[width=\linewidth]{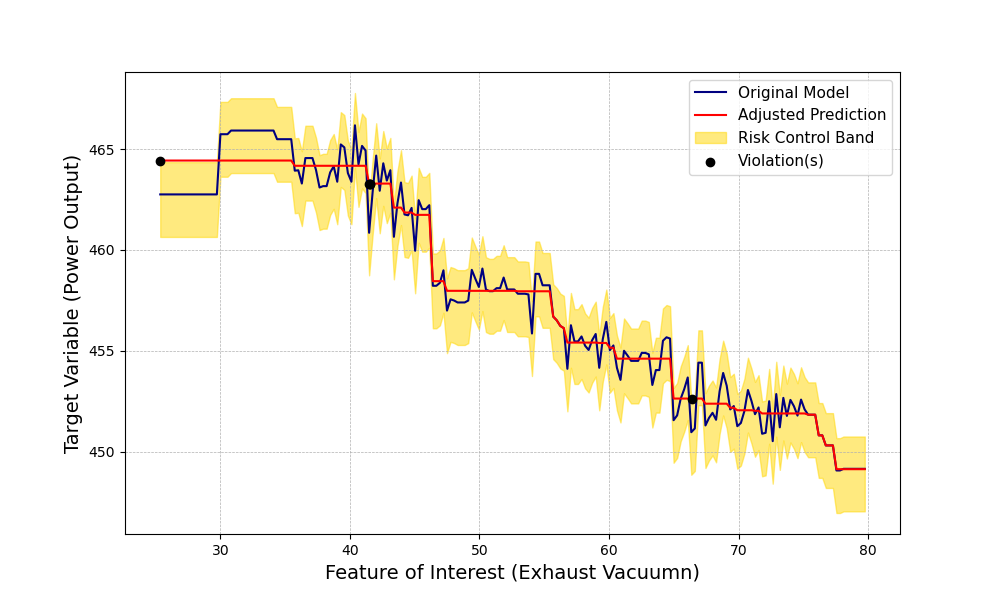}
    \caption{Univariate partial dependence plot of unconstrained model. Risk control band for $\alpha=0.05$. Dashed line exemplifying Theorem 1 demonstrating existence of monotonically decreasing function falling within the conformal band on $0.975 > 1-\alpha$ fraction of the domain. }
    \label{fig:powdashed}
\end{figure}

\subsection{Concavity}

Concavity and convexity are crucial behaviors in many applications. In this context, we focus on concavity in a single feature. A common example where users might expect concave behavior is in recommendation or preference prediction models. According to economic theory, the utility function with respect to the quantity of an item is often quasi-concave, reflecting the principle of diminishing marginal utility \citep{utility}. \citet{neural_utility} propose a novel loss function to account for this expected concavity, which aligns the model with the concavity property $\cP$ during training. Here we again consider aligning a pre-trained model, not trained to satisfy $\cP$, using a proximity oblivious tester $\cT$ for $\cP$ as described in Algorithm \ref{alg:POT_concavity}.

\begin{algorithm}[H]
\caption{POT \(\mathcal{T}\) for property \(\mathcal{P}\) of concavity in dimension \( k \)}
\label{alg:POT_concavity}
\begin{algorithmic}[1]
\STATE Sample \( X_{\text{mid}} \sim \mathcal{D} \)
\STATE Sample \( \delta_{\text{left}}, \delta_{\text{right}} \) from empirical differences in feature \( k \)
\STATE Set \( X_{\text{left}} \) by decreasing feature \( k \) of \( X_{\text{mid}} \) by \( \delta_{\text{left}} \)
\STATE Set \( X_{\text{right}} \) by increasing feature \( k \) of \( X_{\text{mid}} \) by \( \delta_{\text{right}} \)
\STATE Query \( F(X_{\text{mid}}) \), \( F(X_{\text{left}}) \), and \( F(X_{\text{right}}) \)
\STATE Compute \( \alpha = \dfrac{X_{\text{right}}[k] - X_{\text{mid}}[k]}{X_{\text{right}}[k] - X_{\text{left}}[k]} \)
\IF{ \( \min F(X_{\text{mid}}) > \alpha \max F(X_{\text{left}}) + (1 - \alpha) \max F(X_{\text{right}}) \) }
    \RETURN \textbf{Reject}
\ENDIF
\RETURN \textbf{Accept}
\end{algorithmic}
\end{algorithm}

Again we can use a calibration dataset to determine the setting of $\bl = (\lambda^+, \lambda^-)$ via conformal risk control where the loss function is generated as in Definition~\ref{def:loss_function}. Here again the set-valued function will be constructed as $F_{\bl}(X) = [f(X) - \lambda^-, f(X) + \lambda^+]$.  We demonstrate running conformal risk control with this loss function on a real-world dataset in Appendix~\ref{concavity-empirics}.

\newpage

\section{A stylized examination of alignment persistence in AI models}

Consider data generated as:
\[
y=g(X) + h(X) + \noise\,,
\]
where $\noise$ is mean-zero noise with variance $\tau^2$ independent of $X$. Here, $h(X)$ is biased noise we want to ignore, aiming to learn only $g(X)$. Consider the case in which experts expect data to follow $g(X)+\textit{unbiased noise}$, but biased noise $h(X)$ can obscure this. 

One potential reason for the presence of biased noise in data could be due to a measurement error of the outcome that is correlated with select features, leading to an incorrectly calculated outcome. A biased measurement error could occur if there is incomplete data and the presence of the incomplete data is correlated with select features in an unexpected, systematic way. Our goal is to understand how this bias affects model behavior when trying to learn $g(X)$ alone.

Given $n$ i.i.d. samples $\{(X_i, Y_i)\}_{i=1}^n$ from the above model, we denote this dataset by $\cD_n$. We use a random feature model:
\begin{equation*}
f_{\rf}(X; \va,\{\vw_j\}_{j\in[N]}) = \frac{1}{\sqrt{N}}\sum_{j\in[N]}a_j\af(\langle X,\vw_j\rangle)\,,
\end{equation*}
where $\va\in\R^N$ are learned weights, and $\{\vw_j\}_{j\in[N]}$ are fixed random weights. The squared loss is minimized by ridge regression:
\begin{equation*}\label{eq:unconstrained-learning}
    \hat{\va}_\lambda = \arg\min_{\va\in\R^N} \sum_{i\in[n]} \left(Y_i-f_{\rf}(X_i)\right)^2 + \lambda\|\va\|_2^2\,.
\end{equation*}
Users expect a model to exhibit a property $\cP$, satisfied by $g(X)$ but not necessarily by $g(X)+h(X)$. We can constrain training to ensure $\cP$. Let
\(
C_{\cP} = \{\va~|~\va\in\R^N~\text{and}~f_{\rf}(X;\va)~\text{satisfies}~\cP\}\,,
\)
yielding a constrained model:
$\hat{\va}_{\lambda,\cP} = \arg\min_{\va\in C_{\cP}} \sum_{i\in[n]} \left(Y_i-f_{\rf}(X_i)\right)^2 + \lambda\|\va\|_2^2$.

Assuming $g$ and $h$ are polynomials with $\deg_g<\deg_h$, and given specific conditions on data size and model parameters, we consider two settings: (i) \textit{Classic:} $d^{\deg_g+\delta}<N<d^{\deg_h-\delta}$, and (ii) \textit{Underspecified:} $N>d^{\deg_h+\delta}$ for a small $\delta>0$. 

In Appendix \ref{sec:rf-theory-dtails}, we utilize results from \citet{misiakiewicz2023six} to derive insights into the impact of model complexity and data size on adherence to $\cP$. In particular, we show that under certain assumptions, including small noise bias and robustness of property $\cP$, the constrained and unconstrained models have zero distance in the classic setting: $\hat{\va}_{\lambda,\cP} = \hat{\va}_{\lambda}$. However, in the underspecified setting, the constrained and unconstrained models will differ, resulting in a non-zero distance: $\hat{\va}_{\lambda,\cP} \neq \hat{\va}_{\lambda}$. This result implies that in the presence of noise bias, the overparameterized models (i.e., underspecified setting) fail to satisfy the property $\cP$, and this cannot be remedied as the data size increases.

\section{Related work}\label{sec:related}

Our paper draws from a broad range of areas, hence we refer the reader to textbooks and surveys in alignment \citep{everitt2018agi, hendrycks2022unsolved, ji2024ai, aisafety}, conformal prediction \citep{angelopoulos2022gentle}, property testing \citep{Ron, Goldreich_2017}, and linearized neural networks \citep{misiakiewicz2023six}.

RLHF \citep{christiano} has been notably effective in aligning LLMs with human values and intentions, as demonstrated by \citep{ouyang2022training}. Our work considers attempts at alignment that generalzies to models without human-interpretable outputs, which has connections to the scalable oversight problem \citep{irving2018ai, christiano2018supervising, wu2021recursively}. Goal misgeneralization  \citep{langosco22a, shah2022goal} has potential connections to the underspecified pipeline \citep{damour2022underspecification} considered in this paper in the sense that models with equivalent performance according to the training metric may differ in some other user-desired property during deployment. One of the main methods of assurance \citep{Batarseh_2021}, which is concerned with assessing the alignment of pre-trained AI systems, is safety evaluations \citep{perez2022discovering, shevlane2023model} meant to assess risk during deployment, which also has connections to our approach.

The work of \citet{mitigating} closely aligns with ours in both methodology and theme, utilizing conformal risk control to reduce LLM hallucinations \citep{hallucinating}. We discuss connections to this work in Appendix ~\ref{sec:hallucination}.

\section{Discussion}

We introduce a method to align pre-trained models with desired user properties using conformal risk control. By post-processing outputs using property dependent loss functions, we provide probabilistic guarantees that conformal intervals contain functions close to the desired set. This allows for alignment without retraining, effective in both generative and non-generative contexts. Future work should extend these techniques to more properties, explore sample complexity and adaptive querying, and potentially apply them to policy functions in MDP settings for RL agent safety guarantees.

\bibliographystyle{plainnat}
\bibliography{references}


\appendix

\section{Property Testing}\label{property-testing-appendix}

In this section, we provide a brief introduction to property testing, drawing upon standard definitions and concepts from the field. Property testing is a framework for designing algorithms that quickly decide whether a large, complex object (such as a function or a graph) possesses a certain property or is \emph{far} from having that property. The goal is to make this decision by inspecting only a small, random portion of the object, thereby significantly reducing the computational effort required compared to examining the entire object.

Our definitions and explanations are based on \cite{Goldreich_2017}, which offers a comprehensive introduction to property testing. We focus on the basic notions essential for understanding property testing as used in this paper, including the concepts of distance between objects, testers, and proximity-oblivious testers.

\subsection{Distance and $\varepsilon$-Farness}

A central concept in property testing is the notion of distance between objects, which allows us to formalize what it means for an object to be \emph{far} from satisfying a property. For functions, this distance is typically measured with respect to a distribution over the input domain.

\begin{definition}[$\varepsilon$-far]
Let \( f: \mathcal{X} \to \mathcal{Y} \) be a function, \( \mathcal{D} \) a distribution over \( \mathcal{X} \), \( \varepsilon > 0 \), and \( \mathcal{P} \) a property (a set of functions). We say that \( f \) is \emph{\( \varepsilon \)-far} from \( \mathcal{P} \) with respect to \( \mathcal{D} \) if the distance \( \delta_{\mathcal{P}, \mathcal{D}}(f) > \varepsilon \), where
\[
\delta_{\mathcal{P}, \mathcal{D}}(f) = \inf_{g \in \mathcal{P}} \delta_{\mathcal{D}}(f, g), \quad \text{and} \quad \delta_{\mathcal{D}}(f, g) = \Pr_{X \sim \mathcal{D}}[f(X) \neq g(X)].
\]
\end{definition}

In other words, a function \( f \) is \( \varepsilon \)-far from \( \mathcal{P} \) if any function \( g \in \mathcal{P} \) agrees with \( f \) on at most a \( 1 - \varepsilon \) fraction of inputs sampled from \( \mathcal{D} \). This notion of distance allows us to quantify how much \( f \) differs from satisfying the property \( \mathcal{P} \).

While many property testing algorithms assume \( \mathcal{D} \) to be the uniform distribution over \( \mathcal{X} \), in some settings, particularly when dealing with arbitrary data distributions, it is natural to consider general distributions. This leads to the \emph{distribution-free} property testing model \cite{Goldreich_2017,distfree}, where algorithms are designed to work with any underlying distribution \( \mathcal{D} \).

\subsection{Testers and Error Types}

Property testing algorithms, or \emph{testers}, aim to distinguish between objects that have a certain property and those that are \( \varepsilon \)-far from it, by querying the object at a small number of locations. Testers can be categorized based on their error probabilities:

\begin{itemize}
    \item \textbf{One-sided error testers}: These testers always accept objects that have the property (zero false negatives), but only reject objects that do not satisfy the property with some probability that may be below 1. 
    \item \textbf{Two-sided error testers}: These testers may err in both directions, accepting objects that are \( \varepsilon \)-far from the property or rejecting objects that have the property, but with bounded error probabilities.
\end{itemize}

Technically such testers are referred to as \emph{oracle machines}. An \emph{oracle machine} is a theoretical computational model, specifically a Turing machine that has access to an \emph{oracle}. The oracle is a black box that can compute certain functions or provide certain information that the machine cannot compute on its own. In our context, the oracle machine $\mathcal{M}$ can query the function $F$ at any point $x \in \mathcal{X}$ to receive $F(x)$ and can obtain samples from the distribution $\mathcal{D}$. It does not have explicit knowledge of the internal structure of $F$ or $\mathcal{D}$ but can interact with them through these oracle queries.

For our purposes in this work we will only need to discuss one-sided error testers, which we can formally define as follows:

\begin{definition}[One-Sided Error Tester]
A \emph{one-sided error tester} for a property \( \mathcal{P} \) is a probabilistic oracle machine \( \mathcal{M} \) that, given a distance parameter \( \varepsilon > 0 \), oracle access to a function \( f : \mathcal{X} \to \mathcal{Y} \), and oracle access to samples from a distribution \( \mathcal{D} \) over \( \mathcal{X} \), satisfies:
\begin{enumerate}
    \item If \( f \in \mathcal{P} \), then \( \Pr[\mathcal{M}^{f, \mathcal{D}}(\varepsilon) = \text{Accept}] = 1 \).
    \item If \( f \) is \( \varepsilon \)-far from \( \mathcal{P} \) with respect to \( \mathcal{D} \), then \( \Pr[\mathcal{M}^{f, \mathcal{D}}(\varepsilon) = \text{Accept}] \leq \frac{1}{3} \).
\end{enumerate}
\end{definition}

The choice of error probability \( \frac{1}{3} \) is arbitrary and can be reduced to any \( \delta \in (0, \frac{1}{2}) \) by repeating the tester multiple times and taking the majority outcome.

\subsection{Proximity-Oblivious Testers}

In many property testing algorithms, the proximity parameter \( \varepsilon \) is used only to determine the number of iterations or samples needed, rather than affecting the core logic of the tester. This observation leads to the concept of \emph{proximity-oblivious testers} (POTs), where the basic testing procedure is independent of \( \varepsilon \).

\begin{definition}[Proximity-Oblivious Tester]
A \emph{proximity-oblivious tester} for a property \( \mathcal{P} \) is a probabilistic oracle machine \( \mathcal{T} \) that operates without knowledge of the proximity parameter \( \varepsilon \) and satisfies:
\begin{enumerate}
    \item If \( f \in \mathcal{P} \), then \( \Pr[\mathcal{T}^{f, \mathcal{D}} = \text{Accept}] =1 \)
    \item If \( f \notin \mathcal{P} \), then \( \Pr[\mathcal{T}^{f, \mathcal{D}} = \text{Reject}] \geq \rho(\delta_{\mathcal{P}, \mathcal{D}}(f)) \), where \( \rho : (0, 1] \to (0, 1] \) is a non-decreasing function called the \emph{detection probability}.
\end{enumerate}
\end{definition}

Note that the probability of rejecting a function increases with its distance from \( \mathcal{P} \), as quantified by the detection probability function \( \rho \).

To convert a POT into a standard tester with error probability \( \frac{1}{3} \), we can repeat the basic testing procedure a sufficient number of times, determined by \( \rho(\varepsilon) \), and accept or reject based on the aggregate outcomes. Making \( O\left( \frac{1}{\rho(\varepsilon)} \right) \) independent calls to \( \mathcal{T} \) and accepting if all calls accept yields a tester that satisfies the standard definition with the desired error probability \cite{pot}.

\subsection{Adaptivity and Non-Adaptivity}

Testers can also be classified based on whether they are adaptive or non-adaptive in their querying strategy:

\begin{itemize}
    \item \textbf{Non-adaptive testers}: The queries made by the tester are determined in advance, based solely on the random coins and explicit inputs, and do not depend on the answers to previous queries.
    \item \textbf{Adaptive testers}: The tester's queries may depend on the answers to previous queries, allowing for potentially more powerful testing strategies.
\end{itemize}

Non-adaptive testers are often simpler and easier to analyze, and many property testing algorithms are designed to be non-adaptive. In the context of POTs, non-adaptivity is particularly advantageous because the tester's behavior remains consistent regardless of the function being tested. In this work we focus only on non-adaptive testers, but exploring the use of adaptive testers is an exciting future direction.

\section{Proofs for Proximity Oblivious Testers}\label{pot-proofs}

We provide proofs that each of the testing algorithms we use are indeed Proximity Oblivious Testers according to  Definition~\ref{def:pot}. However we note that these properties are well studied in the $f:\cX \to \cY$ case and thus we provide these proofs in the set-valued $f: \cX \to 2^{\cY}$ case mainly for completeness. 

\subsection{Proximity Oblivious Tester for constant property}\label{pot-proofs-const}

For the sake of convenience we repeat the setting from the Example in Section~\ref{property-testing}.

Consider functions $f: \mathbb{R}^d \to \mathbb{R}$, and let $\mathcal{P}$ denote the property that $f$ is constant in the $k$-th dimension. Assume $\cD$ is the empirical distribution of the inputs $X \in \mathbb{R}^d$ for some fixed dataset.

Restrict to set-valued functions $F$ that output compact and connected intervals of the form $[a,b] \subseteq \mathbb{R}$ for $a,b \in \mathbb{R}$. The candidate POT $\mathcal{T}^{F, \mathcal{D}}$ for whether such a set-valued function $F$ accommodates $\cP$ is then as follows: sample $X, X' \sim \mathcal{D}$,  If $F(X) \cap F(X') \neq \varnothing$, then Accept; otherwise, Reject. We claim that this is indeed a POT.

\begin{proof}
Assume that \( F \) accommodates \( \mathcal{P} \). This means there exists a function \( g \in \mathcal{P} \) such that for all \( X \in \mathcal{X} \), \(
g(X) \in F(X).\) Since \( g \) is constant in the \( k \)-th dimension, there exists a function \( h: \mathbb{R}^{d-1} \to \mathbb{R} \) such that for all \( x_k \in \mathbb{R} \) and \( x^{-k} \in \mathbb{R}^{d-1} \), \(g(x_k, x^{-k}) = h(x^{-k}). \) And in the algorithm we sample \( X = (x_k, x^{-k}) \) from \( \mathcal{D} \) and we independently sample \( x_k' \) from the marginal distribution of \( x_k \) and set \( X' = (x_k', x^{-k}) \). Then since \( x^{-k} \) is the same for both \( X \) and \( X' \), and \( g \) is constant in \( x_k \), we have:
\[
g(X) = h(x^{-k}) = g(X').
\]

Because \( g(X) \in F(X) \) and \( g(X') \in F(X') \), it follows that:
\[
h(x^{-k}) \in F(X) \cap F(X').
\]

Therefore, \( F(X) \cap F(X') \neq \emptyset \), and the algorithm returns Accept.

Now suppose that \( F \) is \( \varepsilon \)-far from \( \mathcal{P} \). This means that for any function \( g \in \mathcal{P} \),
\[
\delta_{\mathcal{D}}(F, g) = \Pr_{X \sim \mathcal{D}}[g(X) \notin F(X)] > \varepsilon.
\]

For each \( g \in \mathcal{P} \), define the set of points where \( g \) does not belong to \( F(X) \):
\[
S_g = \{ X \in \mathcal{X} : g(X) \notin F(X) \}.
\]
Although the sets \( S_g \) may vary with \( g \), for each \( g \in \mathcal{P} \), we have \(\mathcal{D}(S_g) > \varepsilon.\) Fix an arbitrary function \( g \in \mathcal{P} \). For a fixed \( x^{-k} \in \mathbb{R}^{d-1} \), define:
\[
p_g(x^{-k}) = \Pr_{x_k \sim \mathcal{D}_{x_k | x^{-k}}}[g(x_k, x^{-k}) \notin F(x_k, x^{-k})],
\]
where \( \mathcal{D}_{x_k | x^{-k}} \) is the conditional distribution of \( x_k \) given \( x^{-k} \).

Since \( \delta_{\mathcal{D}}(F, g) = \mathbb{E}_{x^{-k}}[p_g(x^{-k})] > \varepsilon \), it follows that the average error probability over \( x^{-k} \) is greater than \( \varepsilon \).

When we sample \( X = (x_k, x^{-k}) \) and \( X' = (x_k', x^{-k}) \) with the same \( x^{-k} \) and independent \( x_k \) and \( x_k' \), the probability that both \( g(X) \notin F(X) \) and \( g(X') \notin F(X') \) is:
\[
\Pr[g(X) \notin F(X) \text{ and } g(X') \notin F(X') \mid x^{-k}] = p_g(x^{-k})^2.
\]

Taking expectation over \( x^{-k} \), we get:
\[
\Pr[g(X) \notin F(X) \text{ and } g(X') \notin F(X')] = \mathbb{E}_{x^{-k}}[p_g(x^{-k})^2] \geq \left( \mathbb{E}_{x^{-k}}[p_g(x^{-k})] \right)^2 > \varepsilon^2.
\]

Under the event where both \( g(X) \notin F(X) \) and \( g(X') \notin F(X') \), we have \( g(X) = h(x^{-k}) \notin F(X) \) and \( g(X') = h(x^{-k}) \notin F(X') \). Therefore, \( h(x^{-k}) \notin F(X) \cap F(X') \).

Since \( g \) is constant in \( x_k \), it assigns the same value \( h(x^{-k}) \) to both \( X \) and \( X' \). No other function \( g' \in \mathcal{P} \) can assign a different value at \( x^{-k} \), because all functions in \( \mathcal{P} \) must be constant in the \( k \)-th dimension. Thus, there is no value \( y \) such that \( y \in F(X) \cap F(X') \) that could correspond to \( g'(X) = y \) and \( g' \in \mathcal{P} \). Hence, \( F(X) \cap F(X') = \emptyset \), and the algorithm returns Reject. 

Thus the algorithm Rejects with probability at least \( \varepsilon^2 \) and therefore the algorithm has detection probability \( \rho(\varepsilon) = \varepsilon^2 \) and satisfies Definition~\ref{def:pot} for being a Proximity Oblivious Tester.
\end{proof}

\subsection{Proximity Oblivious Tester for monotonic property}\label{pot-proofs-mon}

We prove that Algorithm~\ref{alg:POT_monotonicity} is indeed a POT for the property $\cP$ of a function being monotonically decreasing in a given dimension $k$. 

\begin{proof}
Assume that \( F \) accommodates \( \mathcal{P} \). Then there exists a function \( g \in \mathcal{P} \) such that for all \( X \in \mathcal{X} \),
\[
g(X) \in F(X).
\]
Since \( g \) is monotonically decreasing in \( x_k \), for any fixed \( x^{-k} \in \mathbb{R}^{d-1} \) and any \( x_k, x_k' \in \mathbb{R} \) with \( x_k \leq x_k' \),
\[
g(x_k, x^{-k}) \geq g(x_k', x^{-k}).
\]

In the algorithm, we sample \( X_1 = (x_1^k, x^{-k}) \) and \( X_2 = (x_2^k, x^{-k}) \), where \( x^{-k} \) is fixed and \( x_1^k, x_2^k \) are independent samples from the marginal distribution of \( x_k \).
First consider the case that \( x_1^k < x_2^k \).
Since \( g \) is monotonically decreasing, we have \( g(X_1) \geq g(X_2) \). Because \( g(X_i) \in F(X_i) \), it follows that
        \[
        \max F(X_1) \geq g(X_1) \geq g(X_2) \geq \min F(X_2).
        \]
Therefore, \( \max F(X_1) \geq \min F(X_2) \), and the condition \( \max F(X_1) < \min F(X_2) \) in step 3 is not satisfied. Now in the second case we have \( x_2^k < x_1^k \). Similarly, \( g(X_2) \geq g(X_1) \). Thus,
        \[
        \max F(X_2) \geq g(X_2) \geq g(X_1) \geq \min F(X_1),
        \]
implying \( \max F(X_2) \geq \min F(X_1) \), so the condition in step 4 is not satisfied.

In both cases, the algorithm does not Reject and therefore Accepts. Thus, if \( F \) accommodates \( \mathcal{P} \), the algorithm always Accepts.

Now suppose \( F \) is \( \varepsilon \)-far from \( \mathcal{P} \). This means that for any \( g \in \mathcal{P} \),
\[
\delta_{\mathcal{D}}(F, g) = \Pr_{X \sim \mathcal{D}}[g(X) \notin F(X)] > \varepsilon.
\]

For each \( g \in \mathcal{P} \), define the set
\[
S_g = \{ X \in \mathcal{X} : g(X) \notin F(X) \}.
\]
While \( S_g \) may vary with \( g \), we have \( \mathcal{D}(S_g) > \varepsilon \) for each \( g \in \mathcal{P} \).

For a fixed \( x^{-k} \in \mathbb{R}^{d-1} \), define
\[
p_g(x^{-k}) = \Pr_{x_k \sim \mathcal{D}_{x_k | x^{-k}}}[g(x_k, x^{-k}) \notin F(x_k, x^{-k})],
\]
where \( \mathcal{D}_{x_k | x^{-k}} \) is the conditional distribution of \( x_k \) given \( x^{-k} \). Then,
\[
\mathbb{E}_{x^{-k}}[p_g(x^{-k})] = \Pr_{X \sim \mathcal{D}}[g(X) \notin F(X)] > \varepsilon.
\]

When we sample \( X_1 \) and \( X_2 \) with the same \( x^{-k} \) and independent \( x_k \) values, the probability that both \( g(X_1) \notin F(X_1) \) and \( g(X_2) \notin F(X_2) \) is
\[
\Pr[g(X_1) \notin F(X_1) \text{ and } g(X_2) \notin F(X_2) \mid x^{-k}] = p_g(x^{-k})^2.
\]

Taking expectation over \( x^{-k} \), we get
\[
\Pr[g(X_i) \notin F(X_i) \text{ for } i=1,2] = \mathbb{E}_{x^{-k}}[p_g(x^{-k})^2] \geq \left( \mathbb{E}_{x^{-k}}[p_g(x^{-k})] \right)^2 > \varepsilon^2,
\]

Under the event where both \( g(X_1) \notin F(X_1) \) and \( g(X_2) \notin F(X_2) \), consider the following cases. In the first case, \( x_1^k < x_2^k \).Since \( g \) is monotonically decreasing, \( g(X_1) \geq g(X_2) \). However, \( g(X_1) \notin F(X_1) \) and \( g(X_2) \notin F(X_2) \). To prevent any \( g' \in \mathcal{P} \) from satisfying \( g'(X_i) \in F(X_i) \), it must be that \( \max F(X_1) < \min F(X_2) \). If \( \max F(X_1) \geq \min F(X_2) \), a function \( g' \) could exist with \( g'(X_i) \in F(X_i) \) and satisfying monotonicity. In the second case, \( x_2^k < x_1^k \), similar reasoning leads to \( \max F(X_2) < \min F(X_1) \). Therefore, when both \( g(X_i) \notin F(X_i) \) and the appropriate ordering of \( x_k \) holds, the algorithm Rejects. 

Since the probability of both \( g(X_i) \notin F(X_i) \) occurring is at least \( \varepsilon^2 \), and the conditions for rejection are met under these events, the algorithm Rejects with probability at least \( \varepsilon^2 \).

Thus, Algorithm~\ref{alg:POT_monotonicity} satisfies the conditions of a Proximity-Oblivious Tester for the property \( \mathcal{P} \) of monotonically decreasing in a given dimension $k$. 

\end{proof}

\subsection{Proximity Oblivious Tester for concave property}

\textbf{Proof:}
Assume that \( F \) accommodates \( \mathcal{P} \); that is, there exists a function \( g \in \mathcal{P} \) such that for all \( X \in \mathcal{X} \), \(g(X) \in F(X).\) Since \( g \) is concave in the \( k \)-th dimension, for any \( X_{\text{left}}, X_{\text{mid}}, X_{\text{right}} \) constructed as in the algorithm, the following inequality holds:
\[
g(X_{\text{mid}}) \leq \alpha g(X_{\text{left}}) + (1 - \alpha) g(X_{\text{right}}),
\]
where
\[
\alpha = \frac{X_{\text{right}}[k] - X_{\text{mid}}[k]}{X_{\text{right}}[k] - X_{\text{left}}[k]}.
\]

Because \( g(X_i) \in F(X_i) \) for \( i \in \{ \text{left}, \text{mid}, \text{right} \} \), we have:
\begin{align*}
\min F(X_{\text{mid}}) &\leq g(X_{\text{mid}}), \\
g(X_{\text{left}}) &\leq \max F(X_{\text{left}}), \\
g(X_{\text{right}}) &\leq \max F(X_{\text{right}}).
\end{align*}

Substituting these into the concavity inequality, we obtain:
\[
\min F(X_{\text{mid}}) \leq \alpha \max F(X_{\text{left}}) + (1 - \alpha) \max F(X_{\text{right}}).
\]

Therefore, the condition in line 7 of the algorithm is not satisfied (i.e., the inequality for rejection does not hold), and the algorithm returns Accept. 

Now suppose that \( F \) is \( \varepsilon \)-far from \( \mathcal{P} \). This means that for any \( g \in \mathcal{P} \),
\[
\delta_{\mathcal{D}}(F, g) = \Pr_{X \sim \mathcal{D}}[g(X) \notin F(X)] > \varepsilon.
\]

For each \( g \in \mathcal{P} \), define the set:
\[
S_g = \{ X \in \mathcal{X} : g(X) \notin F(X) \}.
\]
While \( S_g \) may vary with \( g \), we have \( \mathcal{D}(S_g) > \varepsilon \) for each \( g \in \mathcal{P} \). The probability that \( g(X_{\text{mid}}) \notin F(X_{\text{mid}}) \) is greater than \( \varepsilon \), and similarly, the probabilities that \( g(X_{\text{left}}) \notin F(X_{\text{left}}) \) and \( g(X_{\text{right}}) \notin F(X_{\text{right}}) \) are each greater than \( \varepsilon \). Thus, by independence of the samples, the joint probability that all three events occur is at least \( \varepsilon^3 \):
\[
\Pr\big[ g(X_i) \notin F(X_i) \text{ for all } i \in \{ \text{left}, \text{mid}, \text{right} \} \big] \geq \varepsilon^3.
\]

Under the event where \( g(X_i) \notin F(X_i) \) for all \( i \in \{ \text{left}, \text{mid}, \text{right} \} \), we will show that no concave function \( g' \in \mathcal{P} \) can satisfy \( g'(X_i) \in F(X_i) \) for all \( i \).

Assume for contradiction that there exists a concave function \( g' \in \mathcal{P} \) such that
\[
g'(X_i) \in F(X_i) \quad \text{for all } i \in \{ \text{left}, \text{mid}, \text{right} \}.
\]

Since \( g' \) is concave in the \( k \)-th dimension, it must satisfy:
\[
g'(X_{\text{mid}}) \leq \alpha g'(X_{\text{left}}) + (1 - \alpha) g'(X_{\text{right}}).
\]

Using the bounds from \( F(X_i) \):
\begin{align*}
g'(X_{\text{mid}}) &\geq \min F(X_{\text{mid}}), \\
g'(X_{\text{left}}) &\leq \max F(X_{\text{left}}), \\
g'(X_{\text{right}}) &\leq \max F(X_{\text{right}}).
\end{align*}

Substituting these into the concavity inequality:
\[
\min F(X_{\text{mid}}) \leq \alpha \max F(X_{\text{left}}) + (1 - \alpha) \max F(X_{\text{right}}).
\]

However, the algorithm's condition in line 7 specifies that:
\[
\min F(X_{\text{mid}}) > \alpha \max F(X_{\text{left}}) + (1 - \alpha) \max F(X_{\text{right}}).
\]

This is a contradiction. Therefore, our assumption is false; no concave function \( g' \in \mathcal{P} \) exists that satisfies \( g'(X_i) \in F(X_i) \) for all \( i \). Since under this event (which occurs with probability at least \( \varepsilon^3 \)) no concave function \( g' \in \mathcal{P} \) can fit within \( F(X_i) \) at the points \( X_{\text{left}}, X_{\text{mid}}, X_{\text{right}} \), the algorithm correctly Rejects.

Therefore, the algorithm Rejects with probability at least \( \varepsilon^3 \), so the algorithm satisfies the conditions of a Proximity-Oblivious Tester for the property \( \mathcal{P} \) of concavity in the $k$-th dimension.

\section{Proof of multi-lambda conformal risk control}\label{multi-lambda}

We emphasize that this claim follows directly from the original proof and result of conformal risk control \citep{angelopoulos2024conformal} and we present it here mainly for completeness given our use of it in this multi-dimensional case. 

\textit{Proof of Proposition \ref{prop:1}.} Our proof closely mimics the original proof of boudned risk found in \citet{angelopoulos2024conformal}.

\begin{proof}

Let 
\[
\hat{R}_{n+1}(\bl) = \frac{(L_1(\bl) + \ldots + L_{n+1}(\bl))}{(n + 1)} \quad \text{and}
\]
\begin{align*}
\boldsymbol{\Lambda}_{\min} &= \text{Min} \left\{ \bl \in \boldsymbol{\Lambda} : \hat{R}_{n+1}(\bl) \leq \alpha \right\} \\
\bhl' &\in  \argmin_{\bl \in \boldsymbol{\Lambda}_{\min}} g(\bl)
\end{align*}

The fact that \(\inf_{\bl} L_i(\bl) = L_i(\bl^{\max}) \leq \alpha\) and the fact that $g$ is striclty increasing implies $\bhl'$ is well-defined almost surely. Since \(L_{n+1}(\bl) \leq B\), we have
\[
\hat{R}_{n+1}(\bl) = \frac{n}{n+1} \hat{R}_n(\bl) + \frac{L_{n+1}(\bl)}{n+1} \leq \frac{n}{n+1} \hat{R}_n(\bl) + \frac{B}{n+1}.
\]
So if $\bl$ is such that
\[
\frac{n}{n+1} \hat{R}_n(\bl) + \frac{B}{n+1} \leq \alpha, 
\]

Then we know $\hat{R}_{n+1}(\bl) \leq \alpha$. So when there exists $\bl \in \boldsymbol{\Lambda}$ such that the the above inequality holds, this implies that either $\bhl' \leq \bhl$ or $\bhl \in \boldsymbol{\Lambda}_{\min}$, which in either case implies $g(\bhl') \leq g(\bhl)$. And if $\frac{n}{n+1} \hat{R}_n(\bl) + \frac{B}{n+1} > \alpha $ for all $\bl \in \boldsymbol{\Lambda}$, then we must have $\bhl =\bl^{\max} \geq \bhl'$. So we have $g(\bhl') \leq g(\bhl)$ almost surely. And since $L_i(\bl)$ is non-increasing with respect to $g(\bl)$ we have

\[
\mathbb{E}\left[L_{n+1}(\bhl)\right] \leq \mathbb{E}\left[L_{n+1}(\bhl')\right].
\]

Let $E = \{L_1, \ldots, L_{n+1}\}$, where  $L_{n+1}(\bl) \mid E \sim \text{Uniform}(E)$ by exchangeability, and $\bhl'$ is constant conditional on $E$. Then with the right-continuity of $L_i$ in each dimension we have
\[
\mathbb{E}\left[L_{n+1}(\bhl') \mid E\right] = \frac{1}{n+1} \sum_{i=1}^{n+1} L_i(\bhl'_i) \leq \alpha.
\]

Then applying the law of total expectation and our inequality from above we have 

\begin{align*}
    \mathbb{E}\left[L_{n+1}(\bhl)\right] \leq \alpha,
\end{align*}
completing the proof. 
\end{proof}

\section{Additional Examples}

\subsection{Additional monotonicity examples}

We present more examples of the conformal alignment procedure for monotonicity from Section \ref{sec:monotonicity}. We keep everything the same in terms of 70-15-15 train-calibrate-test split, averaging over 10 runs, and using XGBoost regression models. All experiments were run on an Apple M1 MacBook Pro. We train one constrained model on the train and calibration sets and one unconstrained model on the training set. We use default parameters without hyperparameter tuning. We use the same POT $\TMT$ to obtain loss functions for applying conformal risk control on the unconstrained model using the calibration set. 

\begin{table}[h!]
  \centering
  \caption{Abalone, $n=4177$. Monotonically increasing on Shell\_weight. $\bl^{\max}=(5,5)$.}
  \begin{tabular}{c@{\hspace{12pt}}l@{\hspace{12pt}}c@{\hspace{12pt}}c@{\hspace{12pt}}c@{\hspace{12pt}}c}
    \toprule
    $\alpha$ & $\bl$ & Metric & Unconstrained & Adjusted & Constrained \\
    \midrule
    & $\lambda^+ = 0.17_{\scriptsize{(\pm 0.10)}}$ & MSE & $5.41_{\scriptsize{(\pm 0.18)}}$ & $5.46_{\scriptsize{(\pm 0.19)}}$ & $5.45_{\scriptsize{(\pm 0.17)}}$ \\
    0.1 & $\lambda^- = 0.20_{\scriptsize{(\pm 0.10)}}$ & Risk & $0.58_{\scriptsize{(\pm 0.08)}}$ & $0.09_{\scriptsize{(\pm 0.002)}}$ & $0.00_{\scriptsize{(\pm 0.00)}}$ \\
    \midrule
    & $\lambda^+ = 0.43_{\scriptsize{(\pm 0.24)}}$ & MSE & $5.41_{\scriptsize{(\pm 0.18)}}$ & $5.63_{\scriptsize{(\pm 0.22)}}$ & $5.45_{\scriptsize{(\pm 0.17)}}$ \\
    0.05 & $\lambda^- = 0.54_{\scriptsize{(\pm 0.21)}}$ & Risk & $0.58_{\scriptsize{(\pm 0.08)}}$ & $0.04_{\scriptsize{(\pm 0.001)}}$ & $0.00_{\scriptsize{(\pm 0.00)}}$ \\
    \midrule
    & $\lambda^+ = 1.03_{\scriptsize{(\pm 0.52)}}$ & MSE & $5.41_{\scriptsize{(\pm 0.18)}}$ & $6.67_{\scriptsize{(\pm 0.37)}}$ & $5.45_{\scriptsize{(\pm 0.17)}}$ \\
    0.01 & $\lambda^- = 1.51_{\scriptsize{(\pm 0.53)}}$ & Risk & $0.58_{\scriptsize{(\pm 0.08)}}$ & $0.01_{\scriptsize{(\pm 0.001)}}$ & $0.00_{\scriptsize{(\pm 0.00)}}$ \\
    \bottomrule
  \end{tabular}
\end{table}

\begin{table}[h!]
  \centering
  \caption{Concrete, $n=1030$. Monotonically increasing on Cement. $\bl^{\max}=(2,2)$. }
  \begin{tabular}{c@{\hspace{12pt}}l@{\hspace{12pt}}c@{\hspace{12pt}}c@{\hspace{12pt}}c@{\hspace{12pt}}c}
    \toprule
    $\alpha$ & $\bl$ & Metric & Unconstrained & Adjusted & Constrained \\
    \midrule
    & $\lambda^+ = 0.02_{\scriptsize{(\pm 0.02)}}$ & MSE & $24.15_{\scriptsize{(\pm 2.11)}}$ & $24.18_{\scriptsize{(\pm 2.11)}}$ & $22.44_{\scriptsize{(\pm 2.13)}}$ \\
    0.1 & $\lambda^- = 0.05_{\scriptsize{(\pm 0.04)}}$ & Risk & $0.47_{\scriptsize{(\pm 0.07)}}$ & $0.07_{\scriptsize{(\pm 0.01)}}$ & $0.00_{\scriptsize{(\pm 0.00)}}$ \\
    \midrule
    & $\lambda^+ = 0.16_{\scriptsize{(\pm 0.11)}}$ & MSE & $24.15_{\scriptsize{(\pm 2.11)}}$ & $24.37_{\scriptsize{(\pm 2.13)}}$ & $22.44_{\scriptsize{(\pm 2.13)}}$ \\
    0.05 & $\lambda^- = 0.40_{\scriptsize{(\pm 0.17)}}$ & Risk & $0.47_{\scriptsize{(\pm 0.07)}}$ & $0.04_{\scriptsize{(\pm 0.001)}}$ & $0.00_{\scriptsize{(\pm 0.00)}}$ \\
    \midrule
    & $\lambda^+ = 1.48_{\scriptsize{(\pm 0.28)}}$ & MSE & $24.15_{\scriptsize{(\pm 2.11)}}$ & $25.45_{\scriptsize{(\pm 2.11)}}$ & $22.44_{\scriptsize{(\pm 2.13)}}$ \\
    0.01 & $\lambda^- = 1.60_{\scriptsize{(\pm 0.25)}}$ & Risk & $0.47_{\scriptsize{(\pm 0.07)}}$ & $0.004_{\scriptsize{(\pm 0.002)}}$ & $0.00_{\scriptsize{(\pm 0.00)}}$ \\
    \bottomrule
  \end{tabular}
\end{table}

\newpage

\subsection{Concavity results}
\label{concavity-empirics}

We report the results of applying the conformal alignment procedure for the property of concacity using Algorithm \ref{alg:POT_concavity} as our POT. Implementation details are the same as for the monotonicity examples in Section \ref{sec:monotonicity} besides this change of POT and our procedure for choosing a specific point in order to measure performance on the accuracy metric, which here we take to be MAE. In this case our choice of the point prediction uses a scaling adjustment factor based on the position relative to the feature value at the maximum prediction, scaling from adding 0 to the predictions at the minimum feature value to adding $\lambda$ to the prediction at the feature value of the maximum prediction, then adding 0 again at the maximum feature value. We re-emphasize, though, that the actual point prediction is not our main concern but rather the risk, which we can see matches the anticipated values well.

\begin{table}[h!]
  \centering
  \caption{Seoul Bike Sharing, $n=8760$. Concave on temperature. $\bl^{\max}=(100,100)$}
  \begin{tabular}{c@{\hspace{12pt}}l@{\hspace{12pt}}c@{\hspace{12pt}}c@{\hspace{12pt}}c}
    \toprule
    $\alpha$ & $\bl$ & Metric & Unconstrained & Conformal \\
    \midrule
    & $\lambda^+ = 15.93_{\scriptsize{(\pm 19.63)}}$ & MAE & $94.09_{\scriptsize{(\pm 2.85)}}$ & $99.58_{\scriptsize{(\pm 4.03)}}$ \\
    0.1 & $\lambda^- = 24.67_{\scriptsize{(\pm 20.78)}}$ & Risk & $0.54_{\scriptsize{(\pm 0.02)}}$ & $0.11_{\scriptsize{(\pm 0.03)}}$ \\
    \midrule
    & $\lambda^+ = 25.40_{\scriptsize{(\pm 31.38)}}$ & MAE & $94.09_{\scriptsize{(\pm 2.85)}}$ & $106.44_{\scriptsize{(\pm 7.67)}}$ \\
    0.05 & $\lambda^- = 38.67_{\scriptsize{(\pm 32.64)}}$ & Risk & $0.54_{\scriptsize{(\pm 0.02)}}$ & $0.06_{\scriptsize{(\pm 0.02)}}$ \\
    \midrule
    & $\lambda^+ = 56.93_{\scriptsize{(\pm 38.30)}}$ & MAE & $94.09_{\scriptsize{(\pm 2.85)}}$ & $122.41_{\scriptsize{(\pm 11.41)}}$ \\
    0.01 & $\lambda^- = 71.00_{\scriptsize{(\pm 36.94)}}$ & Risk & $0.54_{\scriptsize{(\pm 0.02)}}$ & $0.02_{\scriptsize{(\pm 0.02)}}$ \\
    \bottomrule
  \end{tabular}
\end{table}

\newpage

\section{Details for a stylized examination of alignment persistence in AI models}\label{sec:rf-theory-dtails}

We begin by describing the stylized model setting that can help us better understand the impact of overparameterization and data quality on a model's ability to satisfy a property.

\paragraph{Notation.} 
We consider feature vectors $X$ to be samples from the uniform distribution on the sphere of radius $\sqrt{d}$ in $\R^d$, that is denoted by $\Sp^{d-1}(\sqrt{d})$. For a function $f$, defined on a domain of the random variable $X$, we define its $L^2$ norm by $\|f\|_{L^2}^2=\E_{X}[f(X)^2]$ where the expectation is with respect to the randomness in $X$. The notation $o_{n, \mathbb{P}}(1)$ for a positive integer means a quantity that goes to zero as $n$ goes to infinity.

\subsection{Noisy data setting}\label{subsec:noise-setting}

\paragraph{Data model.}
Consider a model where the data is generated according to the following equation,
\[
y=g(X) + h(X) + \noise\,,
\]
where $\noise$ has mean $0$ and variance $\tau^2$ and is independent of $X$. Assume that $h(X)$ represents `bad' data that is undesired and we do not want to learn. That is, we only want to learn $g(X)$, but there is a presence of biased noise, $h(X)$, in the data generating model. 

Relating this to an alignment setting, domain experts expect the data to be generated according to $g(X)+\textit{unbiased noise}$, such that $g(X)$ exhibits some desired property $\cP$/ However, the data is generated according to $g(X)+\textit{biased noise}$, potentially obscuring the desired behavior. One potential reason for the presence of biased noise in data could be due to a measurement error of the outcome that is correlated with select features, leading to an incorrectly calculated outcome. A biased measurement error could occur if there is incomplete data and the presence of the incomplete data is correlated with select features in an unexpected, systematic way (e.g., it could hypothetically be correlated to the type of hospital or location at which a patient is treated).

The goal of these analyses is then to examine the impact of the noise bias term on a model’s ability to describe a behavior that is upheld by $g(X)$, but not necessarily $g(X)+h(X)$, across model complexity. For example, $g(X)$ might be monotonically increasing in the first coordinate of $\vx$, but the summation of $g(X)+h(X)$ is not always. We will then try to draw conclusions about the distance between two kinds of models: one that is trained regularly without additional constraints, and another where we add a constraint to the training that forces the model to replicate the behavior of $g(X)$, despite $h(X)$. In the stylized model, when referring to the `distance' between two models, we consider it to be zero if their learned parameters are the same, and non-zero if their learned parameters are different, when trained on the same data.

\paragraph{Learning method.}  We will assume we have access to a data set of $n$ iid samples $\{(X_i,Y_i)\}_{i =1}^n$ obtained from the above data generating model, i.e., $Y_i = g(X_i)+h(X_i)+\noise_i$. We denote this data set by $D_n$.

The function class we use to fit the data is the random feature model, which can be thought of as a two-layer neural network, with $N$ hidden nodes, with the weights of the first layer set randomly, while the second layer weights are learned. This means the function class can be written as:
\begin{equation}
f_{\rf}(X; \va,\{\vw_j\}_{j\in[N]}) = \frac{1}{\sqrt{N}}\sum_{j\in[N]}a_j\af(\langle X,\vw_j\rangle)\,,
\end{equation}
where $\va\in\R^N$ denotes the second layer weights and $\af(\cdot):\R\to\R$ is a non-linear activation function. The random weights $\{\vw_j\}_{j\in[N]}$ are such that each $\vw_j$ is an iid sample from a fixed distribution. For brevity, we drop the reference to $\va$ and $\vw_j$'s and use $f_{\rf}(X)$ or $f_{\rf}(X;\va)$.

For the stylized model analysis, we consider a random feature model, because this function class has been studied in recent theory of deep learning models, and has shown some of the specific properties of modern overparameterized deep learning models such as concepts of benign overfitting and double descent \cite{mei2022generalization, misiakiewicz2023six}. We can thereby use this recent literature as a theoretical foundation to explain the potential impact of model complexity on model's ability to satisfy a property in large train set sizes.

Interestingly, the squared loss function for this problem is convex and we can minimize it by ridge regression. Mathematically, this means for any regularization parameter, $\lambda$, we will search for a vector of real numbers $\hat{\va}_\lambda\in\R^N$ such that,
\begin{equation}\label{eq:unconstrained-learning}
    \hat{\va}_\lambda = \arg\min_{\va\in\R^N} \sum_{i\in[n]} \left(Y_i-f_{\rf}(X_i)\right)^2 + \lambda\|\va\|_2^2\,.
\end{equation}

\subsection{RF theory}\label{subsec:rf-theory}

Assume a certain property, denoted by $\cP$ such that the function $g(X)$ satisfies $\cP$. However, assume that $h(X)$ is such that $g(X)+h(X)$ does not satisfy $\cP$. 

Assume there is a way of constraining the training of $f_{\rf}(X)$ to satisfy the $\cP$. More formally, assume $C_{\cP}\subset\R^N$ is the set of all coefficients $\va$ such that $f_{\rf}(X;\va)$ satisfies $\cP$, i.e.,
\[
C_{\cP} = \{\va~|~\va\in\R^N~\text{and}~f_{\rf}(X;\va)~\text{satisfies}~\cP\}\,.
\]
Then we can consider the following constrained version of the model:
\begin{equation}\label{eq:unconstrained-learning}
    \hat{\va}_{\lambda,\cP} = \arg\min_{\va\in C_{\cP}} \sum_{i\in[n]} \left(Y_i-f_{\rf}(X_i)\right)^2 + \lambda\|\va\|_2^2\,.
\end{equation}

Now assume that the functions $g$ and $h$ are polynomials such that $\deg_g<\deg_h$. Further assume that $n>d^{\deg_h+\delta}$, where $\delta$ is some small positive constant (e.g., $\delta=0.1$) and that $d$ is sufficiently large. Also assume that the activation function $\sigma$ upholds the following genericity conditions (as stated and assumed in Theorem 3 in Section 4.2 of \cite{misiakiewicz2023six}): 1) $|\sigma(x)|\leq c_0\exp(c_1|x|)$ for constants $c_0, c_1>0$; 2) for any $k\geq 0$, $\sigma$ has a non-zero Hermite coefficient $\mu_k(\sigma):= \mathbb{E}_G\{\sigma(G)\text{He}_k(G)\}\neq 0$ (where $G \sim \normal(0,1)$ and $\text{He}_k$ is the $k$-th Hermite polynomial, with the standard normalization $\langle \text{He}_j, \text{He}_k\rangle_{L^2(\normal(0,1))}=k!\mathbbm{1}_{\{j=k\}}$).

To analyze the impact of model complexity on a model's adherence to $\cP$, we consider two settings for the stylized model:
\begin{itemize}
    \item {\bf Classic:} $d^{\deg_g+\delta}<N<d^{\deg_h-\delta}$.
    \item {\bf Underspecified:} $N>d^{\deg_h+\delta}$.
\end{itemize}

To formalize our results, we will in part rely on Theorem 3 in Section 4.2 of \cite{misiakiewicz2023six}. This theorem illustrates the roles that the number of hidden nodes, $N$, and the number of data points, $n$, play in determining the test error that can be achieved by a random feature model when approximating a function. Given the assumptions made in their work, in the setting where $N\ll n$, the number of hidden nodes, $N$, will limit the test error, while if $n\ll N$, the number of data points, $n$, will limit the test error \cite{misiakiewicz2023six}. Building on this intuition, we first state an informal implication of this theorem  \cite{misiakiewicz2023six} applied to our stylized model: under the same assumptions made in their work, for a sufficiently large $d$, in the classic setting, there exists a range of regularization parameters $[0,\lambda^*]$ such that for all $\lambda$ in that range, $f_{\rf}(X;\hat{\va}_\lambda)$ will learn a polynomial of at least $\deg_g$, but strictly less than $\deg_h$. But in the underspecified setting, $f_{\rf}(X;\hat{\va}_\lambda)$ learns both parts $g(X)$ and $h(X)$, for all $\lambda$ in a different range of regularization parameters $[0,\lambda^*]$, meaning it learns the noise bias. This means, in a large data domain, the behavior of the classic and underspecified setting is different. 

In order to formalize this intuition for the underspecified setting, we need to state an assumption on the degree of violation of property $\cB$ around the function $g+h$.

\begin{assumption}[Property Unsatisfied in a Neighborhood]\label{assump:behavior-robust-underspecified}
There exists a positive constant $R$ such that, for all $\va'$ with $\|(g+h)-f_{\rf}(\cdot;\va')\|_{L^2}\le R$, the function $f_{\rf}(X,\va')$ does not satisfy the property $\cP$.
\end{assumption}

\begin{theorem}[Non-vanishing Distance in Underspecified Setting]
    Assume $N>d^{\deg_h+\delta}$, $n>d^{\deg_h+\delta}$, $\max(N/n, n/N)\geq d^\delta$ for a constant $\delta>0$, Assumption \ref{assump:behavior-robust-underspecified} holds, $\sigma$ upholds the genericity conditions stated above, and $d$ is sufficiently large. Then, in the underspecified setting, there exists a constant $\lambda^*$, such that for any $\lambda\in [0,\lambda^*]$ the following holds:
    \begin{equation}
        \hat{\va}_{\lambda,\cP} \neq \hat{\va}_{\lambda}\,.
    \end{equation}
\end{theorem}

\begin{proof}
    Using Theorem 3 of \cite{misiakiewicz2023six}, we have that
    \begin{align*}
    \|(g+h)-f_{\rf}(\cdot;\hat{\va}_{\lambda})\|_{L^2}^2 &= 
    o_{d,\mathbb{P}}(1)\cdot(\|g+h\|_{L^{2+\eta}}^2+\tau^2).
    \end{align*}
    for all $\eta>0$. With $d$ sufficiently large, the square root of the quantity on the right hand side will be less than $R$, which combined with Assumption $\ref{assump:behavior-robust-underspecified}$, means $f_{\rf}(\cdot;\hat{\va}_{\lambda})$ does not satisfy $\cP$, and therefore, $\hat{\va}_{\lambda}\notin C_{\cP}$.
\end{proof}


Under two additional assumptions, we can further state a result showing that in the classic setting, for a given range of regularization parameters, we will be able to learn the property $\cP$. At a high level, these assumptions are that 1) the overall impact of $h$ is small, for example, when it only impacts a small portion of the data; and 2) the property $\cP$ is preserved in a small neighborhood of $g$ among the function class learned in the classic setting. Formalizing the assumptions and stating the result:
\begin{assumption}[Small Noise Bias]\label{assump:small-noise}
Assume there is a small constant $\epsilon_h$ such that 
\[
\|h\|_{L^2} \le \epsilon_h\,.
\]
\end{assumption}

\begin{assumption}[Property Robustness in Function Class]\label{assump:behavior-robust}
For any fixed $N$, 
there exists a positive constant $R_N$ such that, 
for all $\va'$ with $\|g-f_{\rf}(\cdot;\va')\|_{L^2}\le R_N$,  the function $f_{\rf}(X,\va')$ satisfies the property $\cP$.
\end{assumption}

\begin{theorem}[Zero Distance in Classic Setting]
    Assume $d^{\deg_g+\delta}<N<d^{\deg_h-\delta}$, $n>d^{\deg_h+\delta}$, $n/N\geq d^\delta$ for a constant $\delta>0$, Assumptions \ref{assump:small-noise}-\ref{assump:behavior-robust} hold, $\sigma$ upholds the genericity conditions, $d$ is sufficiently large, and $R_N> 2 \epsilon_h$. Then, in the classic setting, there exists a constant $\lambda^*$, such that for any $\lambda\in [0,\lambda^*]$, the following holds:
    \begin{equation}
        \hat{\va}_{\lambda,\cP} = \hat{\va}_{\lambda}\,.
    \end{equation}
\end{theorem}
\begin{proof}
Using the triangle inequality and Theorem 3 of \cite{misiakiewicz2023six} again, we have that
\begin{align*}
\|f_{\rf}(\cdot;\hat{\va}_{\lambda})-g\|_{L^2}^2 &\le 
2(\|f_{\rf}(\cdot;\hat{\va}_{\lambda})-g-h\|_{L^2}^2+
\|h\|_{L^2}^2)\\
&\le 4\|h\|_{L^2}^2\\
&\le 4\epsilon_h^2\,.
\end{align*}
Taking the square root, the right hand side of the last term is less than $R_N$ which combined with Assumption \ref{assump:behavior-robust} means $f_{\rf}(\cdot;\hat{\va}_{\lambda})$ satisfies property $\cP$.
%
\end{proof}
 

\begin{remark}
Note that the constant $R_N$ will converge to zero as $N$ grows, meaning the assumption $R_N>2\epsilon_h$ will break down when $N$ is large. For example, in the underspecified setting, this assumption breaks down and hence the distance does not vanish.
\end{remark}

\subsection{Synthetic validation under a relaxed assumption on the noise bias, and conformal correction}

In this secton, we first numerically validate the above theoretical results, in a setting that noise bias only impacts a small fraction of the pre-training data. Next, we apply our proposed procedure of alignment through conformal risk control. We take the desired property $\cP$ to be monotonicity and use the same procedure as on the real-world datasets in Section \ref{sec:monotonicity}. 

We consider $g(X) = \langle X, \beta \rangle$ with $\beta$ a constant in $\mathbb{R}^d$ with $\beta_1 > 0$ and $\|\beta\|_2 =1$. Then let $h(X)=\min(M,\alpha Z_p x_1^4)$ where $Z_p$ is a Bernoulli random variable with a success probability $p$, that is sampled independently for each data point, and $\alpha$ and $M$ are positive constants. This function is aimed to model a biased noise that is capped by magnitude $M$ and only applies to a fraction $p$ of the population. This is to mimic characteristics of a real dataset where only a portion of the data might have the bias (i.e., `bad' data). This is therefore capturing a more realistic case of the stylized model setting, where the measurement error may only occur for a small subset of the data.

Using a large train set size, we fit the data generated according to the outcome model above using the random feature model with $\lambda$ tuned via cross-validation, assuming the ReLU activation function, i.e.,  $\af(x)=\max(0,x)$, and $\vw_j$ is iid $\normal(0,1)$. We then assess how model complexity affects whether the predicted function is monotone in $x_1$ by looking at the univariate partial dependence plot of the predicted function versus the (counterfactual) unbiased outcome, $g(X)$, for the first coordinate $x_1$. We vary model complexity by $N$, the number of hidden nodes in the model.

Applying this synthetic simulation for specific parameters, Figure \ref{fig:stylized-synthetic-sub2} shows that when training a relatively simple model ($N=5$) on the generated data, the model describes a monotonically increasing relationship. From this we conclude that if we had constrained the model to uphold monotonicity in $x_1$, it would not have changed the final model's parameters. Hence, the distance between the constrained and unconstrained model would be zero and we require no conformal correction. In contrast, as seen in Figure \ref{fig:stylized-synthetic-sub3}, a more complex model ($N=5000)$ predicts a non-monotone relationship, and therefore, if we had constrained the model to uphold monotonicity in $x_1$, it would have changed the final model's parameters, such that the distance between the constrained and unconstrained model would be non-zero and we do require a conformal correction to post-process and align this model.

The synthetic simulation seems to validate that with a seemingly large training dataset, increasing model complexity, on the one hand, can enable high performance, but can also make ML models more sensitive to small noise bias in the data, which can subsequently lead to model behavior inconsistent with requirements by the end-user.

\begin{figure}[ht]
\centering
\includegraphics[width=\linewidth]{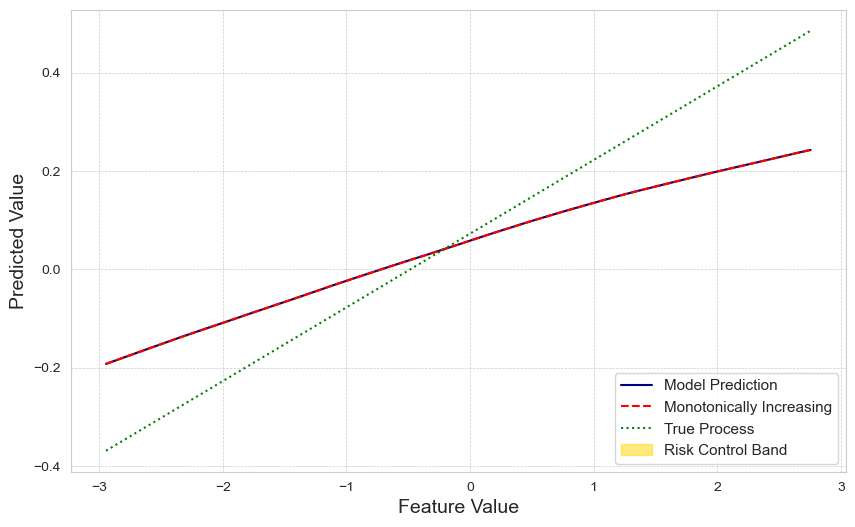}
\caption{Random Feature model (N=5).}
\label{fig:stylized-synthetic-sub2}
\end{figure}

\begin{figure}[ht]
\centering
\includegraphics[width=\linewidth]{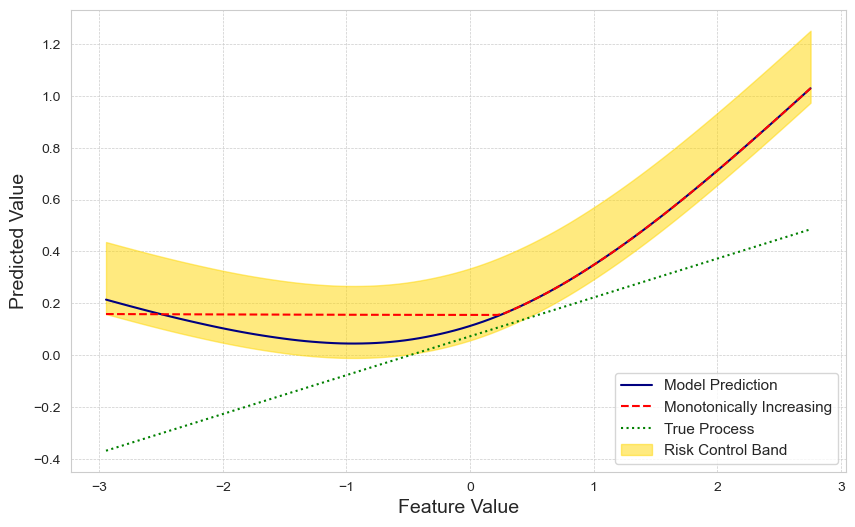}
\caption{Random Feature model (N=5000).}
\label{fig:stylized-synthetic-sub3}
\end{figure}

\section{Connection to Mitigating LLM Hallucinations}\label{sec:hallucination}
Although it is not the main focus of this paper, we show in this section how our general methodology connects to the work of \cite{mitigating}. They examine a potentially random LLM \( f:\cX \to \mathcal{Y} \) and assume access to a confidence score function \( g:\cX \to \mathbb{R} \). They use an abstention function \( a:\Lambda \times \cX \to \{0,1\} \) that decides whether the model should abstain from answering where $a_\lambda(X)=1$ if $g(X) < \lambda$ and $a_\lambda(X)=0$ if $g(X) \geq \lambda$ . The match function \( m:\cX \times \mathcal{Y} \times \mathcal{Y} \to \{0,1\} \) is used to determine if a response \( Y' \) is semantically equivalent to \( Y \), indicating hallucination if not.


\citet{mitigating} use conformal risk control to find an optimal $\hat{\lambda}$ such that the pair $(a_{\hl},f)$ hallucinates at a rate less than $\alpha$. The loss function they use is $\ell:\cX \times \mathcal{Y} \times \Lambda \to \mathbb{R}$ such that $\ell$ punishes failure to abstain when the response does not match the label, that is

\begin{align}
    \ell(X,Y;\lambda) = (1-a_{\lambda}(X))(1-m(X;f(X),Y)) . \label{eq:hallucinating_loss}
\end{align}

In this case we do not change the outputs of the function $f$ itself, instead we use conformal risk control to change the outputs of $a$. Consider the property $\cP$ to be that $(a,f)$ does not hallucinate. A tester for $\cP$ is then simply to query $(a,f)$ on some input $X$ and Reject if $f$ hallucinates or Accept otherwise. Now when $g(X) < \lambda$, we have $C_{\lambda}(X) = \{f(X), \text{Abstain}\}$ meaning that $(a,f)$ can avoid hallucinating, so the tester will Accept. 

\citet{mitigating} use a separate conformal procedures for setting $\lambda$ and a parameter $\beta$ with which they define $m(X;Y',Y)= \mathbb{I}(s(X;Y',Y) \geq \beta)$, where the similarity function $s:\cX \times \cY \times \cY \to \mathbb{R}$ is assumed to be given. We can use our multi-lambda conformal risk control with $\bl=(\lambda_a,\lambda_m)$ to set these parameters simultaneously using a single calibration set $\{(X_i,Y_i)\}_{i=1}^n$ of ground truth query-response pairs sampled from the true data distribution $\cD$. Let

\begin{align*}
    \ell(X,Y;\bl) &= (1-a_{\lambda_a}(X))(1-\mathbb{I}(s(X;f(X),Y) \geq 1-\lambda_m))
\end{align*}

Then we can take $g(\bl)=c_1\lambda_a + c_2 \lambda_m$ for $c_1,c_2 \geq 0$ and choose $\bl$ by our multi-lambda conformal risk control preocedure. As required, this loss function is non-increasing with respect to $\boldsymbol{\Lambda}$ and $g(\bl)$. 

\end{document}